\documentclass[sn-mathphys,Numbered]{sn-jnl}

\usepackage{graphicx}
\usepackage{array}  
\usepackage{multirow}
\usepackage{amsmath,amssymb,amsfonts}
\usepackage{amsthm}
\usepackage{bbm} 
\usepackage{mathrsfs} 
\usepackage[title]{appendix}
\usepackage{xcolor}
\usepackage{textcomp}
\usepackage{manyfoot} 
\usepackage{booktabs}
\usepackage{longtable} 
\usepackage{algorithm}
\usepackage{algorithmicx} 
\usepackage{algpseudocode} 
\usepackage{subcaption}
\usepackage{listings}
\usepackage{tikz} 
\theoremstyle{thmstyleone}%
\newtheorem{theorem}{Theorem}
\theoremstyle{thmstyletwo}%

\theoremstyle{thmstylethree}%

\begin{document}

\title[Article Title]{Continuous Sweep for Binary Quantification Learning}


\author*[1]{\fnm{Kevin} \sur{Kloos}}\email{k.kloos@fsw.leidenuniv.nl}

\author[1]{\fnm{Julian D.} \sur{Karch}}\email{j.d.karch@fsw.leidenuniv.nl}

\author[2]{\fnm{Quinten A.} \sur{Meertens}}\email{q.a.meertens@uva.nl}

\author[3]{\fnm{Sander} \sur{Scholtus}}\email{s.scholtus@cbs.nl}

\author[1]{\fnm{Mark} \sur{de Rooij}}\email{rooijm@fsw.leidenuniv.nl}

\affil[1]{\orgdiv{Faculty of Social and Behavioural Sciences, Institute of Psychology, Methodology and Statistics Unit}, \orgname{Leiden University}, \orgaddress{\street{Wassenaarseweg 52}, \city{Leiden}, \postcode{2233 AK}, \country{The Netherlands}}}

\affil[2]{\orgdiv{Amsterdam School of Economics, Center for Nonlinear Dynamics in Economics and Finance}, \orgname{University of Amsterdam}, \orgaddress{\street{Roetersstraat 11}, \city{Amsterdam}, \postcode{1018 WB}, \country{The Netherlands}}}

\affil[3]{\orgdiv{Department of Methodology}, \orgname{Statistics Netherlands}, \orgaddress{\street{Henri Faasdreef 312}, \city{Den Haag}, \postcode{2492 JP}, \country{The Netherlands}}}


\abstract{A quantifier is a supervised machine learning algorithm, focused on estimating the class prevalence in a dataset rather than labeling its individual observations. We introduce Continuous Sweep, a new parametric binary quantifier inspired by the well-performing Median Sweep, which is an ensemble method based on Adjusted Count estimators. We modified two aspects of Median Sweep: 1) using parametric class distributions instead of empirical distributions for the true and false positive rate; 2) using the mean instead of the median of a set of Adjusted Count estimates. These two modifications allow for a theoretical analysis of the bias and variance of Continuous Sweep. Furthermore, the expressions of bias and variance can be used to define optimal decision boundaries of the set of Adjusted count estimates to be used in the ensemble. We show in three simulation studies that Continuous Sweep outperforms the quantifiers in the group \emph{Classify, Count, and Correct}, including Median Sweep, and is competitive with the two best quantifiers from the group \emph{Distribution Matchers}. Also an empirical data set is analysed with these quantifiers showing similar performances. }



\maketitle

\section{Introduction} \label{sec:introduction}
\noindent Quantification Learning is a machine learning task with the goal of estimating the prevalence of classes in a test set \cite{gonzalez2017quantification, gonzález2017, esuli2023learning}. This task shares similarities with classification where the training data consists of a sample of observations, each with a set of predictor variables denoted by $X$ and a single categorical target variable denoted by $y$. In the test set, only the predictor variables $X$ are available, and the target variable $y$ is unknown. In classification, the training data is used to find a classification rule that aims to predict each $y$ as accurately as possible. In contrast, quantification learning focuses on predicting the distribution of the target variable $y$ in the test set, rather than focusing on the individual labels. Since the target $y$ is a categorical variable, the distribution of $y$ is characterized by the prevalences of the classes. In this paper, we focus on binary target variables. 
 
An example of binary quantification learning is the estimation of the proportion of rooftops equipped with solar panel installations \cite{curier2018monitoring}. Satellites capture the rooftops, which are then processed into a dataset of images, where the pixels in the image serve as predictor variables. The target variable is binary, indicating the presence or absence of a solar panel on each rooftop. The training data set is a set of rooftop images ($X$), each manually verified for the presence or absence of a solar panel ($y$). For the test set, a manual verification of the rooftops is not available. With the training set, a model that predicts the distribution of solar panels is trained. The estimated classification model, or rule, is applied to the test set to obtain an accurate estimate of the prevalence of solar panels. In this exemplary application, such an estimate is important for the government to evaluate its policy on renewable energy, without the need to manually verify the presence of a solar panel on each building. In fact, quantification learning is useful in many disciplines. In epidemiology, for example, the prevalence of COVID-19 in the population determined how governments responded to the pandemic \cite{rath2020, lakhan2020}. In political science, polls and sentiment analysis are used to predict the prevalence of votes in presidential elections \cite{chaudhry2021}. In ecology, the prevalence of coral reef and plankton populations in oceans \cite{beijbom2015} provide essential information on the state of nature.

A naive way to estimate the prevalence in the test set is to count the predicted positive outcomes applied on the test set, better known as the Classify and Count method \cite{forman2005}. The predicted positive outcomes result from an underlying classifier that is trained on a set of training data. However, due to classification errors, only counting positive outcomes will lead to (sometimes severely) biased results. We include an example to illustrate this bias. Assume the classifier can identify a rooftop with solar panels with $98\%$ accuracy and those without solar panels with $97\%$ accuracy. Consequently, $98\%$ of the rooftops with solar panels are classified correctly, but also $3\%$ ($100\%-97\%$) of the rooftops without solar panels are misclassified as having solar panels. These errors affect prevalence estimates, even if the accuracy is high. Consider a test set from a small city with $1000$ houses, for which (in this case) it is known that $100$ have solar panels installed. Applying the classification rule and aggregating the classifications to obtain a prevalence estimate, we predict $125$ ($98\%$ of $100$ plus $3\%$ of $900$) rooftops with solar panels, which overestimates the actual prevalence count of $100$ by $25$ percent. This bias does not disappear if the number of observations in the test set increases. If a larger city has $10,\!000$ houses, out of which $1000$ have solar panels, we predict $1250$ ($98\%$ of $1000$ plus $3\%$ of $9000$) rooftops with solar panels. The bias in prevalence estimates that occurs due to misclassification errors is called \emph{misclassification bias}. 

One of the earliest treatments of misclassification bias was by Bross \cite{bross1954}, who elaborated on the impact of systematic misclassifications of disease diagnoses on chi-square tests. Another more traditional way of estimating the prevalence in the test set is by using the training data as a source of correctly labelled observations, ignoring the predictor variables. If we assume that the training and test data follow the same joint distribution, both data sets should have, apart from some sampling variance, the same prevalence \cite{dougherty2013}. A simple baseline method for estimating the prevalence in the test data, would be to compute the prevalence in the training data. Computing the prevalence in the training data is a trivial task by aggregating the known labels. Often, however, the training data is small and therefore the baseline method has a large variance. Moreover, the training and test set might come from different distributions and then the baseline method gives a biased prevalence estimate. 

In quantification learning, the prevalences of the training and test data are allowed to differ, but it is assumed that the relationship between the target variable and predictors is equal in the training and test set. This assumption is called \textit{prior probability shift} \cite{moreno2012}. An example where prior probability shift plays a role is the case where a classifier for solar panels is trained on data from a city with a prevalence of $10 \%$ and the resulting classification rule is applied on test data from other cities with possibly different prevalences. In this example, the same type of images and the same classifier are used, such that the $X \rightarrow y$ relationship remains the same. Prior probability shift causes the training set not to be representative for the test set, and as a result the prevalence in the training set cannot be used as an unbiased prevalence estimate for the test set. Consequently, more sophisticated quantification methods are needed to accurately estimate the prevalence of a test set. Forman \cite{forman2005} was the first to address the quantification task in the context of machine learning. He proposed several quantification techniques to estimate the type of support needed at a customer service desk (target variable) using technical support logs (predictor variables) \cite{forman2006}. Since then, different types of quantifiers have been developed.

Quantifiers can be categorized into three groups \cite{gonzález2017}. The first group, \emph{Classify, Count and Correct}, contains quantifiers that correct biased counts from the Classify and Count method by using estimates of the misclassification rates from the training data. The second group, \emph{Direct Learners}, contains quantifiers that adapt traditional classification algorithms for quantification tasks, like adapted decision trees (quantification trees) \cite{milli2013}, adapted nearest-neighbors algorithms \cite{barranquero2013}, and adapted support vector machines \cite{esuli2015}. The third group, \emph{Distribution Matching}, contains quantifiers which are primarily focused on tuning the mixture between marginal class distributions to obtain the best estimate of the full distribution. Examples of this third group are the Saerens-Latinne-Decaesteker algorithm (SLD) \cite{saerens2002}, Distribution y-Similarity (DyS) \cite{maletzke2019} and King \& Lu’s algorithm \cite{king2008}. In this paper, we will primarily focus on quantifiers that belong to the first group.

Forman elaborated on three quantifiers within the group of Classify, Count and Correct \cite{forman2008}. First, Classify and Count is the simplest quantifier that counts and averages the estimated labels, which generally results in biased prevalence estimates. Second, Adjusted Count corrects the Classify and Count estimate by removing the bias using estimated true and false positive rates of the training data. A similar correction was also proposed earlier by Kuha and Skinner \cite{kuha1997}. Third, Median Sweep is an ensemble method that computes the median of a set of Adjusted Count estimates \cite{forman2008}. See Section \ref{sec:prev} for a more detailed treatment. 

Recently, Schumacher et al. \cite{schumacher2021} compared 24 quantifiers from all three groups in an extensive simulation study. The results show that, for binary quantification problems, Median Sweep is the best quantifier from the group Classify, Count and Correct. Furthermore, Median Sweep performs better than most quantifiers from the other two groups. The DyS algorithm \cite{maletzke2019} is competitive to the Median Sweep, for some data sets the Median Sweep outperforms the DyS, for other data sets the DyS outperforms Median Sweep. In the recent LeQua2022 competition, a competition designed for quantification tasks, the Distribution Matching algorithms DyS and SLD performed particularly well \cite{esuli2022detailed}. Unfortunately, the Median Sweep was not included in LeQua2022. Also in \cite{esuli2023learning}, it was found that DyS and SLD perform well.

In their review of quantification learning, \cite{gonzález2017} suggest that ``more solid theoretical analyses are needed to better understand not only the behavior of these algorithms but also the learning problem in general''. Since then, several key theoretical contributions have been made in the field of quantification learning, including derivations of analytic expressions for bias and variance for various Adjusted Count quantifiers \cite{kloos2021, kloos2021a, tasche2017, tasche2021, meertens2021}. Hitherto, however, no theoretical results have been derived for  Median Sweep. Two elements of Median Sweep complicate a theoretical analysis. First, Median Sweep uses empirical cumulative density functions, that is, step functions, to estimate the true and false positive rate. Second, Median Sweep computes the median of a set of adjusted count estimates as the final estimate. Our aim is to make a contribution to the theoretical results, by slightly adjusting, and preferably improving, the construction of Median Sweep.


To these ends, we present a new, binary, parametric quantifier inspired by Median Sweep, and call it Continuous Sweep. We modify two aspects of Median Sweep. First, Continuous Sweep estimates the true and false positive rate using parametric distributions instead of the empirical cumulative density function. Second, Continuous Sweep uses the mean instead of the median, to compute its final prevalence estimate. Changing these two elements enables us to perform a theoretical analysis of Continuous Sweep and derive analytical expressions for its bias and variance.

Continuous Sweep has been constructed in such a way that analytic expressions for its bias and variance can be derived without losing the key properties of Median Sweep, that is, Continuous Sweep remains an ensemble quantifier in the group \emph{Classify, Count, and Correct} that aggregates a range of adjusted count estimates. In addition, we use the analytic expressions for bias and variance to further improve the quantifier. Most importantly, whereas Median Sweep uses a heuristic value to determine which set of Adjusted Count estimates is used in the final estimation, from our analytic expressions we derive an optimal value that minimizes the variance of the estimate. 

After introducing Continuous Sweep and deriving the expressions for bias and variance, we compare the performance of Continuous Sweep to other quantifiers. Using a simulation study and an application, we compare the performance of Continuous Sweep with Median Sweep, that is, the best quantifiers from the group \emph{Classify, Count, and Correct}, and with the best quantifiers from the group \emph{Distribution Matching}, namely DyS and SLD. We will show that Continuous Sweep generally outperforms quantifiers from its own group, while it is competitive with the distribution matching quantifiers. 

The remainder of the paper is organized as follows. Section $2$ revisits previous work in the \emph{Classify, Count, and Correct} quantifiers and introduces notation that will be used throughout the paper. In Section $3$, we introduce Continuous Sweep and provide analytic expressions for its bias and variance. Furthermore, we show how these expressions can be used to find an optimal set of adjusted count estimates. In Section $4$, we compare Continuous Sweep with other quantifiers in three simulation studies. In Section $5$, we compare Continuous Sweep with other quantifiers on the LeQua2022 competition data. We conclude the paper in Section $6$ with an overview of the main results and some discussion topics.

\section{Previous Work in Classify, Count, and Correct Quantifiers} \label{sec:prev}
\noindent This section provides an overview of the relevant literature on quantification learning, specifically focusing on the first group of \textit{Classify, Count, and Correct} methods. We reiterate the definitions of the quantifiers Classify and Count, Adjusted Count, and Median Sweep, including the bias and variance of the first two. Furthermore, we introduce the notation that will be used in the remainder of this paper.

Binary quantification learning is a semi-supervised machine learning task that aims to estimate the prevalence of positive labeled observations $\alpha$ in an unlabeled data set $D_\text{test}$, where the unlabeled data set $D_\text{test}$ is used for a quantifier that is trained with a labeled data set $D_\text{train}$ \cite{gonzález2017}. In binary quantification tasks, each observation $i$ in $D_\text{train}$ consists of a feature vector $x_i \in \mathbb{R}^p$ and a class label $y_i \in \{-1, +1 \}$, while observations in $D_\text{test}$ only consist of an p-dimensional feature vector $x_i \in \mathbb{R}^p$. The missing labels in $D_\text{test}$ are predicted using a discriminant function $\hat\delta(x_i)$ of a classifier that is trained on training set $D_\text{train}$. The discriminant function maps the feature vector $x_i$ to a real-valued discriminant score: $\hat\delta(x_i): \mathbb{R}^p \rightarrow \mathbb{R}$. Subsequently, the quantifier labels observation $i$ in $D_\text{test}$ either positively or negatively, depending on whether $\hat\delta(x_i)$ is larger than a threshold $\theta$, that is 
\begin{align}
    \hat y_i = 
    \begin{cases} 
        +1 & \text{if $\hat\delta(x_i) \geq \theta$} \\
        -1 & \text{if $\hat\delta(x_i) < \theta$}.
    \end{cases}
\end{align}
The most straightforward approach to estimate the prevalence of positives in $D_\text{test}$ is to count the number of observations assigned to the positive class. In the literature, this approach is called \emph{Classify and Count} \cite{forman2005}. Given that $D_\text{test}$ is of size $n_\text{test}$, the prevalence as estimated by the quantifier \emph{Classify and Count}, denoted $\hat\alpha_\text{CC}$, is defined as

\begin{align} \label{eq:cac}
    \hat\alpha_\text{CC} &= \frac{1}{n_\text{test}} \sum_{i = 1}^{n_\text{test}} \mathbb{I}\left\{ \hat y_i = +1 \right\}, 
\end{align}
where $\mathbb{I}$ denotes the indicator function. Generally, a soft classifier uses a threshold probability of $0.5$ to classify an observation as positive, however, other thresholds could also be used.

The expected value of Classify and Count is determined by adding the number of correctly specified positives (true positives) and the incorrectly specified negatives (false positives) as

\begin{align}
    E\left[ \hat\alpha_\text{CC} \right] &= E \left[ \frac{1}{n_\text{test}} \sum_{i = 1}^{n_\text{test}} \mathbb{I}\left\{ \hat y_i = +1 \right\} \right] \nonumber \\ 
    &= \alpha \cdot p^+ + (1- \alpha) \cdot (1 - p^-) \nonumber \\
    &=  \alpha \cdot (p^+ + p^- - 1) + 1 - p^-, \label{eq:ev-cac}
\end{align}
where $p^+$ and $p^-$ denote the true positive rate $P(\hat y_i = +1 \mid y_i = +1)$ and the true negative rate $P(\hat y_i = -1 \mid y_i = -1)$, respectively \cite{kuha1997, scholtus_accuracy_2020}. The bias of Classify and Count is the difference between the expected value and the true value $\alpha$, that is,

\begin{align}
    B\left[ \hat\alpha_\text{CC} \right] &=  E\left[ \hat\alpha_\text{CC} \right] - \alpha \nonumber \\
    &=  \alpha \cdot (p^+ + p^- - 2) + 1 - p^-. \label{eq:bias-cac}
\end{align}
We see that Classify and Count potentially has a large bias caused by the difference between false positives and false negatives \cite{gonzalez2017quantification}. 

The variance of Classify and Count is heavily dependent on the size of the test set \cite{scholtus_accuracy_2020}, and is denoted as

\begin{align}
    V\left[ \hat\alpha_\text{CC} \right] &= \frac{\alpha p^+ (1 - p^+) + (1-\alpha) p^- (1-p^-)}{n_{\text{test}}}.
\end{align}


Eqs. \eqref{eq:ev-cac} and \eqref{eq:bias-cac} show mathematically that Classify and Count is rarely unbiased. Given a constant $p^+$ and $p^-$, Classify and Count is only unbiased for a single value $\alpha_0 = \frac{1-p^-}{2-p^+-p^-}$ , but is biased for all other values of $\alpha$ \cite{gonzalez2017quantification}. Therefore, it can be concluded that the Classify and Count quantifier cannot be used reliably under a moderate to large prior probability shift and that the bias caused by misclassifications needs to be corrected. 

\emph{Adjusted Count} corrects the Classify and Count quantifier using estimated true positive and negative rates $\hat p^+$ and $\hat p^-$ from $D_\text{train}$ by rewriting Eq. \eqref{eq:ev-cac} in terms of $\hat\alpha_\text{CC}$ as

\begin{align}
    E\left[ \hat\alpha_\text{CC} \right] &= \alpha \cdot (p^+ + p^- - 1) + 1 - p^- \nonumber \\
    \alpha \cdot (p^+ + p^- - 1) &= E\left[ \hat\alpha_\text{CC} \right] - (1 - p^-) \nonumber \\
    \alpha &= \frac{E\left[ \hat\alpha_\text{CC} \right] - (1 - p^-)}{p^+ - (1 - p^-)}. \label{eq:ev-cac-rewritten}
\end{align}
Because we assume only prior probability shift, the distribution of $p^+$ and $p^-$ in $D_\text{train}$ and $D_\text{test}$ are identical and, as a result, the estimated rates $\hat p^+$ and $\hat p^-$ can be used as plug-in estimators to estimate the prevalence $\alpha$ in an unbiased way. Moreover, the parameters $\hat p^+$, $\hat p^-$, and $\hat\alpha_\text{CC}$ all depend on the threshold $\theta$. A change in $\theta$ can result in a different value for $\hat\alpha_\text{CC}$ because more observations in $D_\text{test}$ pass the decision rule $\hat\delta(x_i) \geq \theta$. Rewriting Eq. \eqref{eq:cac} more precisely results in

\begin{align}
    \hat \alpha_\text{CC}(\theta, D_\text{test}) &= \frac{1}{n_\text{test}} \sum_{i = 1}^{n_\text{test}} \mathbb{I}\left\{ \hat\delta(x_i) \geq \theta \right\}. \label{eq:cac-theta} 
\end{align}
The training data $D_\text{train}$ are used to estimate the true positive rate $p^+$ and the true negative rate $p^-$ for a given threshold $\theta$. The true positive rate can be estimated by the proportion of observations with $y = 1$ that pass the decision rule $\hat\delta(x_i) \geq \theta$ in $D_\text{train}$, and the true negative rate can be estimated by the proportion of observations with $y = -1$ that do not pass the decision rule in $D_\text{train}$. In mathematical notation, this boils down to

\begin{align}
    \hat p^+(\theta, D_\text{train}) &= \frac{\sum_{i=1}^{n_\text{train}} \mathbb{I}\left\{ y_i = +1 \right\} \cdot \mathbb{I}\left\{ \hat\delta(x_i) \geq \theta \right\}}{\sum_{i=1}^{n_\text{train}} \mathbb{I}\left\{ y_i = +1 \right\}}, \label{eq:tpr-theta} \\
    \hat p^-(\theta, D_\text{train}) &= \frac{\sum_{i=1}^{n_\text{train}} \mathbb{I}\left\{ y_i = -1 \right\} \cdot \mathbb{I}\left\{ \hat\delta(x_i) < \theta \right\}}{\sum_{i=1}^{n_\text{train}} \mathbb{I}\left\{ y_i = -1 \right\}} \label{eq:tnr-theta},
\end{align}
where $n_\text{train}$ denotes the sample size of the training set $D_\text{train}$. By combining Eqs. \eqref{eq:cac-theta}, \eqref{eq:tpr-theta} and \eqref{eq:tnr-theta}, Eq. \eqref{eq:ev-cac-rewritten} can be rewritten in terms of $\hat\alpha_\text{CC}$ as

\begin{align}
    \hat\alpha_\text{AC}(\theta, D_\text{train}, D_\text{test}) &= \frac{\hat\alpha_\text{CC}(\theta, D_\text{test}) - (1 - \hat p^-(\theta, D_\text{train}))}{\hat p^+(\theta, D_\text{train}) - (1 - \hat p^-(\theta, D_\text{train}))}. \label{eq:ac-theta}
\end{align}
The Adjusted Count is an unbiased prevalence estimate \cite{buonaccorsi_measurement_2010}. The variance of Adjusted count, as derived by \cite{kloos2021, kuha1997}, is

\begin{align}
    V\left[ \hat\alpha_\text{AC} \right] &= \frac{\alpha p^+ (1-p^+) + (1-\alpha)p^-(1-p^-)}{n_\text{test} (p^+ + p^- - 1)^2}. \label{eq:var-ac}
\end{align}
From Eq. \eqref{eq:var-ac}, we observe that the Adjusted Count quantifier has a large variance if $p^+ + p^-$ is close to $1$. This happens when the values of the true and false positive rates are close to each other. Therefore, a sophisticated choice of $\theta$ can decrease the variance of the adjusted count quantifier. Many quantification algorithms use the Adjusted Count quantifier and differ only by their choice of $\theta$. For example, the quantification method \emph{MAX} chooses $\theta$ such that the difference between $\hat p^+$ and $\hat p^-$ is the largest and method \emph{T50} chooses $\theta$ where $\hat p^+ = 0.50$ \cite{forman2008}. Similar Adjusted Count quantifiers exist in the field of survey methodology. In the 1990s, Kuha and Skinner \cite{kuha1997} provided an overview of the most relevant correction methods, which were later extended by \cite{kloos2021, meertens2021, kloos2021a}. 

The \textit{Median Sweep} quantifier is an ensemble technique that starts with Adjusted Count estimates (i.e., using Eq. \eqref{eq:ac-theta}) for each $\theta \in \left\{ \hat\delta(x): x \in D_\text{test} \right\}$
for which the difference between $\hat p^+$ and $(1 - \hat p^-)$ is larger than $\frac{1}{4}$. This results in a set of estimates using Adjusted Count where outliers (i.e., values outside these boundaries) are discarded. The final prevalence estimate is obtained as the median of this set of estimates \cite{forman2006}. More formally, from the set $X(D_\text{test}) = \left\{ \hat\delta(x): x \in D_\text{test} \right\}$, we only consider the elements of the set $\boldsymbol{\theta}_\text{MS} = \left\{ \theta \in X(D_\text{test}): \hat p^+(\theta, D_\text{train}) + \hat p^-(\theta, D_\text{train}) - 1 > \frac{1}{4} \right\}$. Using this set $\boldsymbol{\theta}_\text{MS}$, Median Sweep is 

\begin{align}
    \hat\alpha_\text{MS} = \text{median} \left( \left\{ \hat\alpha_\text{AC}(\theta, D_\text{train}, D_\text{test}): \theta \in \boldsymbol{\theta}_\text{MS} \right\} \right).
\end{align}\label{eq:ms}
From Schumacher et al. \cite{schumacher2021}, we know that Median Sweep is, on average over a variety of datasets, the best quantifier within the group of Classify, Count and Correct, and that it performs better than most quantifiers from the other two groups (Direct Learners and Distribution Matchers). However, it is unknown under which specific circumstances Median Sweep performs well. Moreover, we believe that Median Sweep could be improved even further, because of its heuristic decision rules for discarding observations where the difference between the true positive and false positive rate is smaller than $\frac{1}{4}$. The threshold of $\frac{1}{4}$ seems rather arbitrary and could potentially be optimized. With those considerations in mind, we introduce a new quantifier named Continuous Sweep.

\section{Continuous Sweep}
\noindent In this section, we introduce the Continuous Sweep quantifier starting from Median Sweep (Section 3.1). In Section 3.2, the bias, variance, and mean squared error of Continuous Sweep are derived. Last, in Section 3.3, we show how to compute the optimal threshold $p^\Delta$ that will replace the arbitrary cutoff value of $\frac{1}{4}$ used in Median Sweep.

\subsection{Constructing the Continuous Sweep Quantifier from Median Sweep}
\noindent The Median Sweep quantifier has three properties that either affect the performance of the quantifier or affect the ability to derive the bias and variance of the quantifier analytically. Median Sweep uses
\begin{enumerate}
    \item step functions for the true positive and negative rate with respect to $\theta$;
    \item the median of a set of Adjusted Count estimates;
    \item a heuristic decision rule for discarding Adjusted Count estimates.
\end{enumerate}
The step functions for the true positive and true negative rate as a function of $\theta$ can be avoided by using parametric estimates for the density of discriminant scores $\hat \delta(x_i)$. For example, if the discriminant scores are assumed to be normally distributed for each class, then $p^+(\theta)$ and $p^-(\theta)$ can be estimated using the continuous cumulative distribution functions, instead of the non-parametric estimated step functions, as illustrated in Fig. \ref{fig:example_disc_cont1}. The first step therefore is fitting a normal distribution to the discriminant scores $\hat \delta(x_i)$ in $D_\text{train}$ (Fig. \ref{fig:example_disc_cont1}). After obtaining the parameters of the fitted distribution, the densities can be integrated to obtain cumulative distribution functions \footnote{The distribution function is of the lower tail.} (Fig. \ref{fig:example_disc_cont2}). To illustrate the difference between the step functions and the parametric estimates, the non-parametric empirical cumulative distribution function is also shown in Fig. \ref{fig:example_disc_cont2}. Instead of interpreting the curves in Fig. \ref{fig:example_disc_cont2} as cumulative distribution functions, the fitted lines can also be considered estimates of the true positive rate $\hat p^+(\theta)$ and the false positive rate $1 - \hat p^-(\theta)$.

The first adjustment to Median Sweep is that the true and false positive rates are estimated using parametric estimates of the cumulative distribution. Therefore, in the equations for Adjusted Counts we need to replace 
$\hat p^+$ by $\hat F^+(\theta, D_\text{train})$ and $1 - \hat p^-$ by $\hat F^-(\theta, D_\text{train})$. In the equation for the expected values and bias, we change $p^+$ by $F^+(\theta)$ and $1 - p^-$ by $F^-(\theta)$.

Then, the equations for the expected value and the bias of the Classify and Count quantifier become

\begin{align}
    E[\hat\alpha_\text{CC}(\theta)] &= \alpha \left( F^+(\theta) - F^-(\theta) \right) + F^-(\theta) \label{eq:ecc-cont} \\
    B[\hat\alpha_\text{CC}(\theta)] &= \alpha \left( F^+(\theta) - F^-(\theta) - 1 \right) + F^-(\theta), \label{eq:bcc-cont}
\end{align}
and accordingly the equation for the Adjusted Count quantifier becomes

\begin{align}
    \hat\alpha_\text{AC}(\theta, D_\text{train}, D_\text{test})
    &= \frac{\hat\alpha_\text{CC}(\theta, D_\text{test}) - \hat F^-(\theta, D_\text{train})}{\hat F^+(\theta, D_\text{train}) - \hat F^-(\theta, D_\text{train})}. \label{eq:ac-theta-cont}
\end{align}

\begin{figure}[h!]
    \centering
    \begin{subfigure}{0.49\textwidth}
        \centering
        \includegraphics[width = \textwidth]{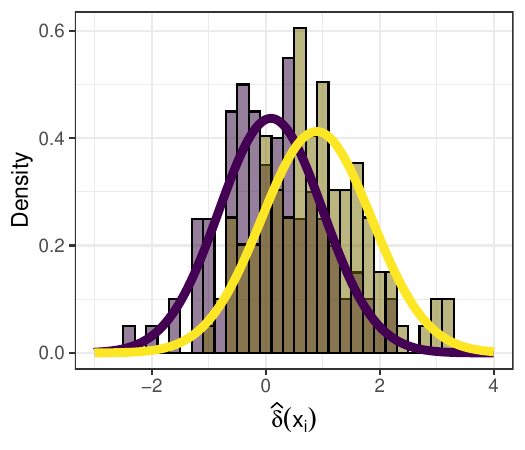}
        \caption{}
        \label{fig:example_disc_cont1}
    \end{subfigure}
    \hfill
    \begin{subfigure}{0.49\textwidth}
        \centering
        \includegraphics[width = \textwidth]{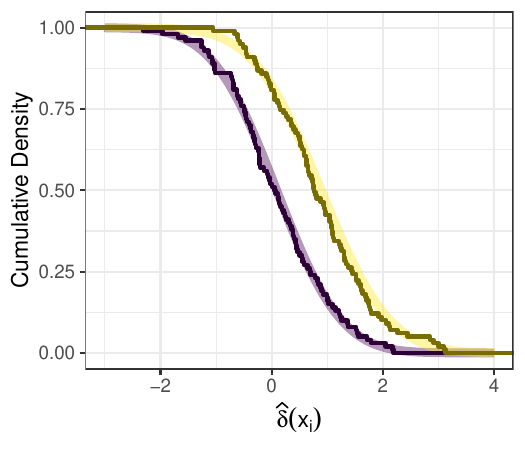}
        \caption{}
        \label{fig:example_disc_cont2}
    \end{subfigure}
    \caption{The difference between Median Sweep and Continuous Sweep in estimating the true and false positive rate. In Figure (a), the histogram shows the distributions of the discriminant scores in the negative (purple) and the positive (yellow) class, where the corresponding coloured lines show the fitted Normal distributions. In Figure (b), Median Sweep uses empirical cumulative distribution functions (i.e., step functions) while Continuous Sweep uses the Gaussian cumulative distribution (i.e., smooth curves) to estimate the true/false positive rate.}
    \label{fig:example_disc_cont}
\end{figure}
Unlike for the functions of the true positive rate and the false positive rate, the function of the Classify and Count quantifier $\hat\alpha_\text{CC}(\theta, D_\text{test})$ is preferably not estimated as a continuous, parametric curve. Despite that the approach from Kloos et. al \cite{kloos2022unileiden} uses kernel methods to estimate the distribution of the Classify and Count quantifier, we in hindsight believe that a step function is preferred over a kernel estimate. The full Classify and Count distribution is namely a mixture distribution with an unknown mixture parameter closely related to prevalence $\alpha$. Therefore, $\hat\alpha_\text{CC}(\theta, D_\text{test})$ remains a step-function for future calculations. 

As illustrated in Fig. \ref{fig:cont-sweep-lines}, $\hat\alpha_\text{AC}$ is a continuous function between any two discriminant scores in the test set $D_\text{test}$. Within such an interval, $\hat\alpha_\text{CC}$ is constant and, consequently, the graph of $\hat\alpha_\text{AC}$ is continuous. For each jump in $\hat\alpha_\text{CC}$, we move to a new interval. For extreme discriminant scores $\theta$, the estimates in $\hat\alpha_\text{AC}$ have a high variance resulting in potentially extreme prevalence estimates. This can be seen in Figure \ref{fig:cont-sweep-ac}, where the gradient of the adjusted count function is steeper for extreme values of $\theta$. Median Sweep suffers from this problem as well where some adjusted count estimates have high variance. Therefore, as described in Section 2, Median Sweep only considers prevalence estimates at thresholds equal to discriminant scores occurring in $D_\text{test}$ for which the difference between $\hat F^+(\theta, D_\text{train})$ and $\hat F^-(\theta, D_\text{train})$ is greater or equal than the cut-off value ($p^\Delta$). Forman \cite{forman2008} used a heuristic value of $p^\Delta = \frac{1}{4}$ as the general cutoff value for Median Sweep. Continuous Sweep inherits the decision rule from Median Sweep, but due to the characteristics of Continuous Sweep and its distribution functions $\hat F^+(\theta, D_\text{train})$ and $\hat F^-(\theta, D_\text{train})$, the decision rule is applied differently. Instead of including or discarding prevalence estimates at only the discriminant scores occurring in $D_\text{test}$, i.e., a discrete set, Continuous Sweep considers all prevalence estimates for every $\theta$ on an entire interval. Therefore, the thresholds $\theta$ where $p^\Delta = \frac{1}{4}$ are defined as decision boundaries $\theta_\text{l}$ (left) and $\theta_\text{r}$ (right), and all the points between the decision boundaries are included in the final prevalence estimate \footnote{For now, we assume that exactly two solutions to the equation exist. If $F^+$ and $F^-$ are normal distributions with means $\mu^+$ and $\mu^-$, this is always the case.}. To obtain the decision boundaries, we solve the equation  $F^+(\theta, D_\text{train}) - F^-(\theta, D_\text{train}) - \frac{1}{4} = 0$. In Fig. \ref{fig:cont-sweep-lines-dots}, the difference between Median Sweep and Continuous Sweep is illustrated, where Median Sweep only includes the green points in the final prevalence estimate, Continuous Sweep includes all values between the two decision boundaries. For now, we compute $\theta_\text{l}$ (left) and $\theta_\text{r}$ using the cut-off value $p^\Delta = \frac{1}{4}$, but this could be easily extended to other values between $0$ and $1$. These decision boundaries give a set of adjusted count estimates where Median Sweep takes the median of the adjusted count estimates as its final prevalence estimate.

\begin{figure}
    \centering
    \begin{subfigure}{0.49\textwidth}
        \centering
        \includegraphics[width=\textwidth]{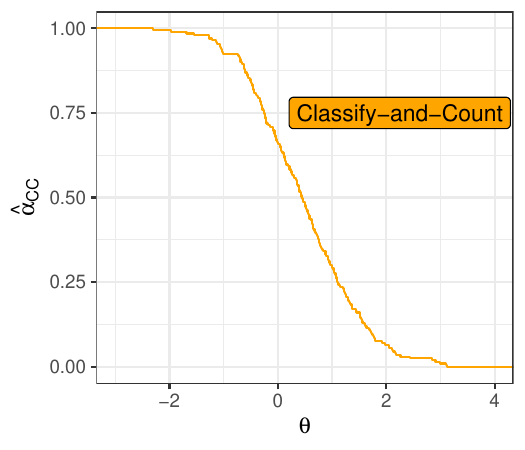}
        \caption{}
        \label{fig:cont-sweep-cac}
    \end{subfigure}
    \hfill
    \begin{subfigure}{0.49\textwidth}
        \centering
        \includegraphics[width=\textwidth]{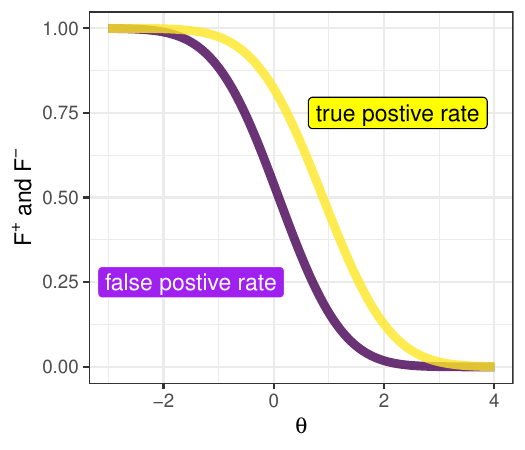}
        \caption{}
        \label{fig:cont-sweep-tpr-fpr}
    \end{subfigure}
    \hfill
    \begin{subfigure}{\textwidth}
        \centering
        \includegraphics[width = \textwidth]{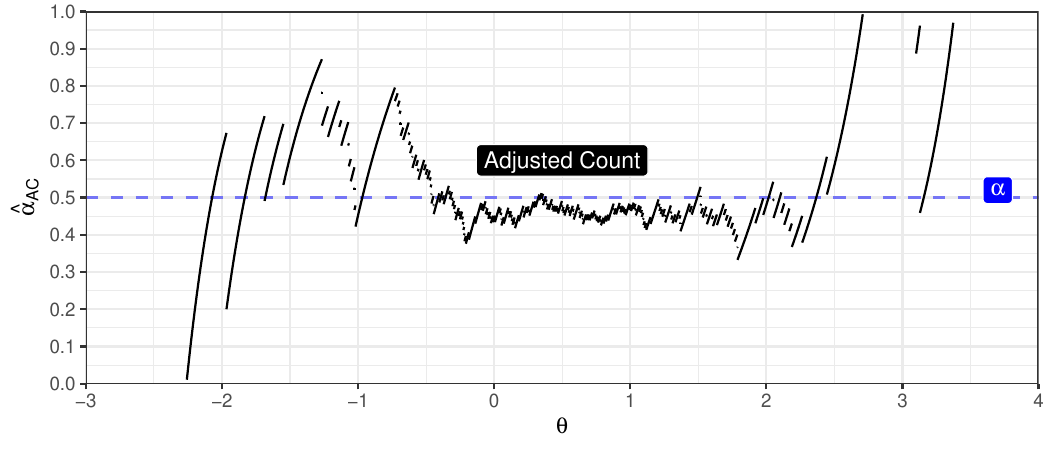}
        \caption{}
        \label{fig:cont-sweep-ac}
    \end{subfigure}
    \caption{Using a Classify and Count function as a step function (Panel (a)) and the true and false positive rate as continuous functions (Panel (b)), the Adjusted Count quantifier can estimate prevalence $\alpha$ at each $\theta$ (Panel (c)).}
    \label{fig:cont-sweep-lines}
\end{figure}

\begin{figure}
    \centering
    \includegraphics[width = \textwidth]{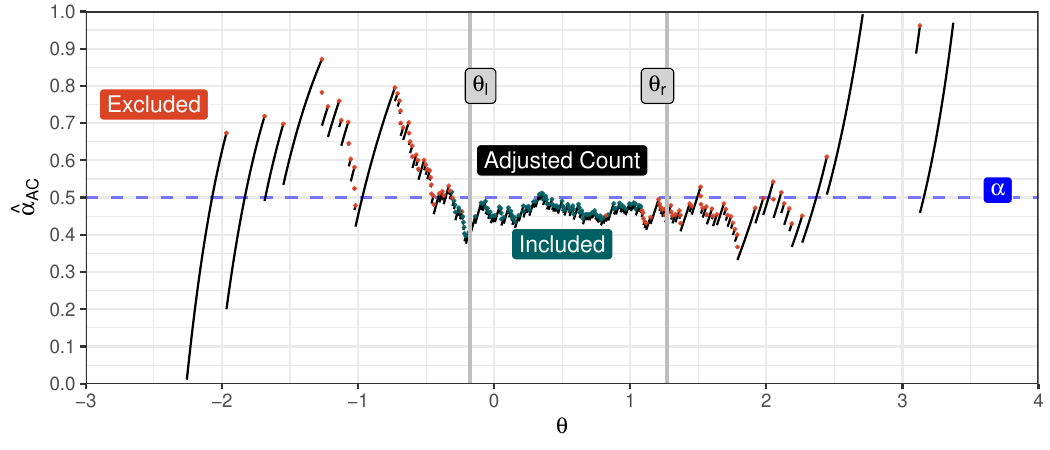}
    \caption{Example of discarding individual data points (red/green) against decision boundaries (grey, vertical). For Median Sweep, the green points are included in the prevalence calculation, while the red points are excluded. Continuous Sweep includes every $\theta$ between $\theta_\text{l}$ and $\theta_\text{r}$ in the calculations.}
    \label{fig:cont-sweep-lines-dots}
\end{figure}

The second adjustment relates to the median. The median can only be computed analytically when the cumulative distribution function of $\hat\alpha_\text{AC}(\theta, D_\text{train}, D_\text{test})$ between $\theta_\text{l}$ and $\theta_\text{r}$ is known. Because $\hat\alpha_\text{AC}(\theta, D_\text{train}, D_\text{test})$ consists of many continuous functions in this parameter space that are collectively non-bijective, it is far-fetched to obtain its cumulative distribution function. As illustrated in Fig. \ref{fig:cont-sweep-area}, computing the mean analytically requires less effort, because it only includes the area under the curve. For every line segment between $\theta_\text{l}$ and $\theta_\text{r}$, this area can be calculated by integrating Eq. \eqref{eq:ac-theta-cont} with respect to $\theta$, resulting in a set of areas. The Continuous Sweep quantifier $\hat\alpha_\text{CS}$ is obtained by summing all areas between $\theta_\text{l}$ and $\theta_\text{r}$ and dividing by the distance between $\theta_\text{l}$ and $\theta_\text{r}$ as
\begin{align}\label{eq:part-cont-sweep}
    \hat\alpha_\text{CS}(D_\text{train}, D_\text{test}) &= 
    \frac{1}{\theta_\text{r} - \theta_\text{l}} \int_{\theta_\text{l}}^{\theta_\text{r}} \frac{\hat\alpha_\text{CC}(\theta, D_\text{test}) - \hat F^-(\theta, D_\text{train})}{\hat F^+(\theta, D_\text{train}) - \hat F^-(\theta, D_\text{train})} d\theta.
\end{align}
Eq. \eqref{eq:part-cont-sweep} cannot be evaluated analytically, because the curve exists of many continuous parts. The algorithm to evaluate Eq. \eqref{eq:part-cont-sweep} is found in Algorithm \eqref{alg:cont-sweep}.
\begin{algorithm}
    \caption{Computing Continuous Sweep}
    \begin{algorithmic}[1]
        \State \textbf{Input:} training data $D_\text{train}$ test data $D_\text{test}$, difference parameter $p^\Delta$.
        \State \textbf{Output:} prevalence estimate $\hat\alpha_\text{CS}$
        \State $\theta_\text{l}, \theta_\text{r} \leftarrow$ \verb|ObtainDecisionBoundaries|($D_\text{train}, p^\Delta$)
        \State \verb|ths| $\leftarrow \verb|sort|(D_\text{test}[ D_\text{test} \geq \theta_\text{l} \text{ \& } D_\text{test} \leq \theta_\text{r}])$
        \State $N_\text{ths} \leftarrow$ \verb|length|(\verb|ths|)
        \State \verb|A| $\leftarrow$ \verb|length|(\verb|ths| $- 1$)
        \For {$i$ in $1, 2, ..., N_\text{ths} - 1$}
            \State \verb|A[i]| $\leftarrow$ \verb|Integrate|($\hat\alpha_\text{AC}(\theta, D_\text{train}, D_\text{test})$\verb|, min = ths[i], max = ths[i+1], var =| $\theta$) 
        \EndFor
        \State \verb|SumArea| $\leftarrow$ \verb|sum|(\verb|A|)
        \State \textbf{return} $\hat\alpha_\text{CS} = \frac{\text{SumArea}}{\theta_\text{r} - \theta_\text{l}}$
    \end{algorithmic}
    \label{alg:cont-sweep}
\end{algorithm}

\begin{figure}
    \centering
    \includegraphics[width = \textwidth]{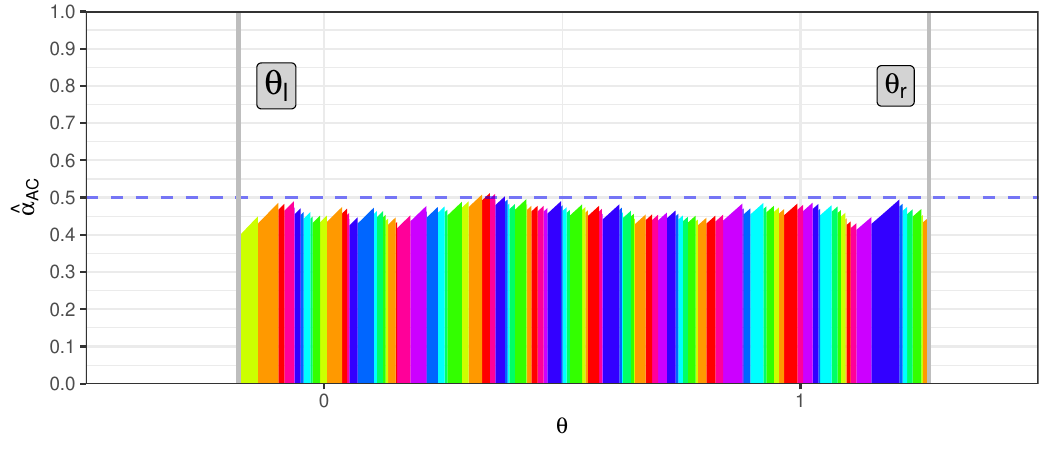}
    \caption{Example of using integrals to compute the prevalence using Continuous Sweep. The estimated prevalence is computed by summing the coloured areas divided by the difference between $\theta_\text{l}$ and $\theta_\text{r}$.}
    \label{fig:cont-sweep-area}
\end{figure}
To these ends, we adapted two characteristics of Median Sweep to create our new quantifier Continuous Sweep. The adaptations enable us to compute derive expressions for the bias and variance of Continuous Sweep. The last adaptation regards changing the heuristic decision rule for discarding Adjusted Count estimates. We can change the heuristic decision rule using the expressions for the bias and variance of Continuous Sweep. In Section \ref{subsec:proofs}, we derive the expressions for the bias and variance and in Section \ref{subsec:pdelta}, we explain how to use these expressions to improve the decision rule.

\subsection{The bias and variance of Continuous Sweep} \label{subsec:proofs}
\noindent One of the main advantages of Continuous Sweep over Median Sweep is the ability to derive equations for the bias and variance, which in turn can be used to optimize the performance of both quantifiers (Section \ref{subsec:pdelta}). In this paper, we derive equations for the bias and variance of Continuous Sweep, given that the marginal distributions $F^+(\theta)$ and $F^-(\theta)$ are known. Given these assumptions, we first prove that Continuous Sweep is unbiased.
\newline
\begin{theorem}\label{thm:ev-ac}
    Assuming that $F^+(\theta)$ and $F^-(\theta)$ are known, the expected value of the Continuous Sweep quantifier $\hat\alpha_\text{CS}$ is equal to the true prevalence $\alpha$ and is therefore unbiased.
    
    \begin{align}\label{eq:ev-cs}
        E[\hat\alpha_\text{CS}(\theta, D_\text{test})] &= \alpha
    \end{align}
\end{theorem}

\begin{proof}
    The expected value of Continuous Sweep is equal to
     \begin{align}\label{eq:ev-cs-wrongsides}
         E \left[ \hat\alpha_\text{CS}(\theta, D_\text{test}) \right] &= E \left[ \frac{1}{\theta_\text{r} - \theta_\text{l}} \int_{\theta_\text{l}}^{\theta_\text{r}} \hat\alpha_\text{AC}(\theta, D_\text{test}) d \theta  \right]  \nonumber \\
         &= \frac{1}{\theta_\text{r} - \theta_\text{l}} E \left[ \int_{\theta_\text{l}}^{\theta_\text{r}} \frac{\hat\alpha_\text{CC}(\theta, D_\text{test}) - F^-(\theta)}{F^+(\theta) - F^-(\theta)} d \theta \right].
     \end{align}
    According to Fubini's Theorem, we may exchange the integral and the expectation of Eq. \eqref{eq:ev-cs-wrongsides} if the following holds:

    \begin{align}\label{eq:fubini-infty}
        E \left[ \int_{\theta_\text{l}}^{\theta_\text{r}} \bigg\lvert \frac{\hat\alpha_\text{CC}(\theta, D_\text{test}) - F^-(\theta)}{F^+(\theta) - F^-(\theta)} \bigg\rvert d \theta \right] < \infty.
    \end{align}
    The decision boundaries ensure that, for all $\theta$ between $\theta_\text{l}$ and $\theta_\text{r}$, $F^+(\theta) - F^-(\theta) \geq p^\Delta$. Furthermore, the parameters $\hat\alpha_\text{CC}(\theta, D_\text{test})$ and $F^-(\theta)$ are naturally bounded between $0$ and $1$, and $p^\Delta$ is a constant independent of $D_\text{test}$. Therefore, $\hat\alpha_\text{AC}(\theta, D_\text{test})$ is bounded between $\frac{-1}{p^\Delta}$ and $\frac{1}{p^\Delta}$ if $p^\Delta > 0$, for all values of $\theta$ and all test sets $D_\text{test}$. Hence, for each $D_\text{test}$, it holds that $\int_{\theta_\text{l}}^{\theta_\text{r}} \bigg\lvert \frac{\hat\alpha_\text{CC}(\theta, D_\text{test}) - F^-(\theta)}{F^+(\theta) - F^-(\theta)} \bigg\rvert d \theta \leq \frac{\theta_\text{r} - \theta_\text{l}}{p^\Delta}$. Therefore, Eq. \eqref{eq:fubini-infty} holds and Fubini's theorem can be applied. As a result, the integral and expectation in Eq. \eqref{eq:ev-cs-wrongsides} are interchangeable as

    \begin{align}\label{eq:prf-ac-ev2}
        E \left[ \hat\alpha_\text{CS}(\theta, D_\text{test}) \right] &= \frac{1}{\theta_\text{r} - \theta_\text{l}} \int_{\theta_\text{l}}^{\theta_\text{r}} E \left[ \hat\alpha_\text{AC}(\theta, D_\text{test}) \right] d \theta.
    \end{align}
    Assuming that the parameters of $F^+$ and $F^-$ for each value of $\theta$ are known, it follows that $D_\text{test}$ is the only stochastic variable in the expectation. Using Eq. \eqref{eq:ecc-cont} then yields

    \begin{align}\label{eq:prf-ac-ev3}
        E \left[ \hat\alpha_\text{AC}(\theta, D_\text{test}) \right] &= \frac{E[\hat\alpha_\text{CC}(\theta, D_\text{test})] - F^-(\theta)}{F^+(\theta) - F^-(\theta)} = \frac{\alpha \cdot (F^+(\theta) - F^-(\theta)) }{F^+(\theta) - F^-(\theta)} = \alpha.
    \end{align}
    Hence, the expectation of the Adjusted Count quantifier is the true prevalence irrespective of $\theta$. Then, it is straightforward to substitute Eq. \eqref{eq:prf-ac-ev3} in Eq. \eqref{eq:prf-ac-ev2}, resulting in
 \begin{align}
        E \left[ \hat\alpha_\text{CS}(\theta, D_\text{test}) \right] &= \frac{1}{\theta_\text{r} - \theta_\text{l}} \int_{\theta_\text{l}}^{\theta_\text{r}} \alpha \text{ } d \theta = \frac{\alpha(\theta_\text{r} - \theta_\text{l})}{\theta_\text{r} - \theta_\text{l}} = \alpha.
    \end{align}
    This proves that $\hat\alpha_\text{CS}$ is unbiased.
\end{proof}
Because Continuous Sweep is unbiased, the mean squared error of the Continuous Sweep is equal to its variance, which we derive below.  
\newline
\begin{theorem}\label{thm:var-ac}
    The variance of Continuous Sweep is
    \begin{align}
        V \left[ \hat\alpha_\text{CS} \right] &= \frac{2}{n_\text{test}(\theta_\text{r} - \theta_\text{l})^2} \int_{y = \theta_\text{l}}^{\theta_\text{r}} \int_{x = y}^{\theta_\text{r}}\frac{ \alpha F^+(y) (1 - F^+(x)) + (1 - \alpha) F^-(y) (1 - F^-(x))}{(F^+(x) - F^-(x)) (F^+(y) - F^-(y))} dxdy. \label{eq:var-cs}
    \end{align}
    
    \begin{proof}
    The first step in computing the variance of Continuous Sweep is rewriting the variance in terms of expectations as

    \begin{align}
        V \left[ \hat\alpha_\text{CS} \right] &= E \left[ (\hat\alpha_\text{CS})^2 \right] - E \left[ \hat\alpha_\text{CS} \right]^2 \nonumber \\
        &= E \left[ \left( \frac{1}{\theta_\text{r} - \theta_\text{l}} \int_{\theta = \theta_\text{l}}^{\theta_\text{r}} \hat\alpha_\text{AC}(\theta, D_\text{test}) \text{ } d\theta \right)^2 \right] - \alpha^2, \label{eq:cs1}
    \end{align}
    by using Theorem \ref{thm:ev-ac} for computing the first moment. The second moment can be computed by expanding its expectation as

    \begin{align}\label{eq:sq-cs}
        E \left[ (\hat\alpha_\text{CS})^2 \right] &= E \left[ \left( \frac{1}{\theta_\text{r} - \theta_\text{l}} \int_{x = \theta_\text{l}}^{\theta_\text{r}}\hat\alpha_\text{AC}(x, D_\text{test}) \text{ } dx \cdot \frac{1}{\theta_\text{r} - \theta_\text{l}} \int_{y = \theta_\text{l}}^{\theta_\text{r}} \hat\alpha_\text{AC}(y, D_\text{test}) \text{ } dy \right) \right] \nonumber \\
        &= \frac{1}{(\theta_\text{r} - \theta_\text{l})^2} E \left[ \left( \int_{x = \theta_\text{l}}^{\theta_\text{r}} \int_{y = \theta_\text{l}}^{\theta_\text{r}} \hat\alpha_\text{AC}(x, D_\text{test}) \cdot \hat\alpha_\text{AC}(y, D_\text{test}) \text{ } dxdy \right) \right] \nonumber \\
    \end{align}
    Similar to Theorem \eqref{thm:ev-ac}, it can be proved that the integrand of Eq. \eqref{eq:sq-cs} is uniformly bounded, because $0 < (\theta_\text{r} - \theta_\text{l})^2 < \infty$ and $0 \leq \hat\alpha_\text{AC}(\theta, D_\text{test}) \leq \frac{1}{(p^\Delta)^2}$ for all $D_\text{test}$ and all $\theta$ in the relevant region. Hence, Fubini's Theorem can be used once more to interchange the integral and the expectation as

    \begin{align}
        E \left[ (\hat\alpha_\text{CS})^2 \right] = \frac{1}{(\theta_\text{r} - \theta_\text{l})^2} \int_{x = \theta_\text{l}}^{\theta_\text{r}} & \int_{y = \theta_\text{l}}^{\theta_\text{r}} E \left[ \hat\alpha_\text{AC}(x, D_\text{test}) \cdot \hat\alpha_\text{AC}(y, D_\text{test}) \right] dxdy.
    \end{align}
    Using the covariance rule $E[XY] = C[X,Y] + E[X]E[Y]$, the integral can be rewritten as an integral with a covariance as
    
    \begin{align}
        E \left[ (\hat\alpha_\text{CS})^2 \right] &= \frac{1}{(\theta_\text{r} - \theta_\text{l})^2} \int_{x = \theta_\text{l}}^{\theta_\text{r}} \int_{y = \theta_\text{l}}^{\theta_\text{r}} C \left[ \hat\alpha_\text{AC}(x, D_\text{test}), \hat\alpha_\text{AC}(y, D_\text{test}) \right] \nonumber \\ & \qquad \qquad \qquad \qquad \qquad \qquad  + E[\hat\alpha_\text{AC}(x, D_\text{test})] E[\hat\alpha_\text{AC}(y, D_\text{test})] dxdy \\
        &= \alpha^2 + \frac{1}{(\theta_\text{r} - \theta_\text{l})^2} \int_{x = \theta_\text{l}}^{\theta_\text{r}} \int_{y = \theta_\text{l}}^{\theta_\text{r}} C \left[ \hat\alpha_\text{AC}(x, D_\text{test}) , \hat\alpha_\text{AC}(y, D_\text{test}) \right] dxdy. \label{eq:cs2}
    \end{align}
    If we substitute Eq. \eqref{eq:cs2} in Eq. \eqref{eq:cs1}, we observe that the terms $\alpha^2$ cancel out. What remains is the covariance function, which can be simplified to
    
    \begin{align}
        V \left[ \hat\alpha_\text{CS} \right] &= \frac{2}{(\theta_\text{r} - \theta_\text{l})^2} \cdot \int_{y = \theta_\text{l}}^{\theta_\text{r}} \int_{x = y}^{\theta_\text{r}} C \left[ \hat\alpha_\text{AC}(x, D_\text{test}), \hat\alpha_\text{AC}(y, D_\text{test}) \right] dxdy . \label{eq:vcs-review}
    \end{align}
    Similar to Theorem \ref{thm:ev-ac}, the parameters of $F^+$ and $F^-$ are assumed to be known and can be written out of the covariance function as
    
    \begin{align}
       V \left[ \hat\alpha_\text{CS} \right] &=  \frac{2}{(\theta_\text{r} - \theta_\text{l})^2} \int_{y = \theta_\text{l}}^{\theta_\text{r}} \int_{x = y}^{\theta_\text{r}} \frac{C \left[ \hat\alpha_\text{CC}(x, D_\text{test}), \hat\alpha_\text{CC}(y, D_\text{test}) \right]}{(F^+(x) - F^-(x)) (F^+(y) - F^-(y))} dxdy. \label{eq:cov-cs-global}
    \end{align}
    The Classify and Count quantifier counts the number of elements in $D_\text{test}$ that are larger or equal than $\theta$. $D_\text{test}$ is constructed as follows. One draws an i.i.d. sample of size $n_\text{test}^+$ from $F^+$, denoted as random variables $A_1, A_2, ..., A_{n^+_{\text{test}}}$. Similarly, one draws another i.i.d. sample of size $n_\text{test}^-$ from $F^-$, denoted as random variables $B_1, B_2, ..., B_{n^-_{\text{test}}}$. We assume that each $A_i$ is independent of each $B_k$. Using indices $x$ and $y$ for a pair of samples, the covariance function can be rewritten as
    
    \begin{align}
       C[ & \hat\alpha_\text{CC}(x, D_\text{test}), \hat\alpha_\text{CC}(y, D_\text{test}) ] = \nonumber \\ & C \left[ \frac{1}{n_\text{test}} \left( \sum_{i=1}^{n_\text{test}^+} \mathbb{I} \left\{ A_i \geq x \right\} + \sum_{k=1}^{n_\text{test}^-} \mathbb{I} \left\{ B_k \geq x \right\} \right), \frac{1}{n_\text{test}} \left(\sum_{j=1}^{n_\text{test}^+} \mathbb{I} \left\{ A_j \geq y \right\} + \sum_{l=1}^{n_\text{test}^-} \mathbb{I} \left\{ B_l \geq y \right\} \right)  \right] \nonumber \\
       =& \frac{1}{n_\text{test}^2} \cdot C \left[ \sum_{i=1}^{n_\text{test}^+} \mathbb{I} \left\{ A_i \geq x \right\} + \sum_{k=1}^{n_\text{test}^-} \mathbb{I} \left\{ B_k \geq x \right\}, \sum_{j=1}^{n_\text{test}^+} \mathbb{I} \left\{ A_j \geq y \right\} + \sum_{l=1}^{n_\text{test}^-} \mathbb{I} \left\{ B_l \geq y \right\} \right]. \label{eq:cov-big1}
    \end{align}
    The cross-products of the covariance function in Eq. \eqref{eq:cov-big1} vanish because of our assumptions on independence. What remains is the equation
    
    \begin{align}
        C[ & \hat\alpha_\text{CC}(x, D_\text{test}), \hat\alpha_\text{CC}(y, D_\text{test}) ] = \nonumber \\ &\frac{1}{n_\text{test}^2} \cdot  C \left[ \sum_{i=1}^{n_\text{test}^+} \mathbb{I} \left\{ A_i \geq x \right\}, \sum_{j=1}^{n_\text{test}^+} \mathbb{I} \left\{ A_j \geq y \right\} \right] + \frac{1}{n_\text{test}^2} \cdot  C \left[ \sum_{k=1}^{n_\text{test}^-} \mathbb{I} \left\{ B_k \geq x \right\}, \sum_{l=1}^{n_\text{test}^-} \mathbb{I} \left\{ B_l \geq y \right\} \right] \nonumber \\
        &= \frac{n_\text{test}^+}{n_\text{test}^2} C \left[ \mathbb{I} \left\{ A_1 \geq x \right\}, \mathbb{I} \left\{ A_1 \geq y \right\} \right] + \frac{n_\text{test}^-}{n_\text{test}^2} C \left[ \mathbb{I} \left\{ B_1 \geq x \right\}, \mathbb{I} \left\{ B_1 \geq y \right\} \right]. \label{eq:cov-big2}  
    \end{align}
    The covariances in Eq. \eqref{eq:cov-big2} can be computed using the general rule $C[X, Y] = E[XY] - E[X]E[Y]$. Applying this rule combined with the indicator functions within the covariance function and the restriction that $x \geq y$ (review Eq. \eqref{eq:vcs-review}) results in the following covariance function for the positives:
    
    \begin{align}
        C \left[ \mathbb{I} \left\{ A_1 \geq x \right\}, \mathbb{I} \left\{ A_1 \geq y \right\} \right] &= E[\mathbb{I} \left\{ A_1 \geq x \right\} \cdot \mathbb{I} \left\{ A_1 \geq y \right\}] - E[\mathbb{I} \left\{ A_1 \geq x \right\}] \cdot E[\mathbb{I} \left\{ A_1 \geq y \right\}] \nonumber \\
        &= E[\text{min}(\mathbb{I} \left\{ A_1 \geq x \right\}, \mathbb{I} \left\{ A_1 \geq y \right\})] - E[\mathbb{I} \left\{ A_1 \geq x \right\}] \cdot E[\mathbb{I} \left\{ A_1 \geq y \right\}] \nonumber \\
        &= F^+(y) - F^+(x) \cdot F^+(y) = F^+(y)(1-F^+(x)). \label{eq:cov-ai}
    \end{align}
    Similarly, the following covariance function for the negatives can be derived:
    
    \begin{align}
        C \left[ \mathbb{I} \left\{ B_1 \geq x \right\}, \mathbb{I} \left\{ B_1 \geq y \right\} \right] &= F^-(y) - F^-(x) \cdot F^-(y) = F^-(y)(1-F^-(x)). \label{eq:cov-bi}
    \end{align}
    Substituting Eqs. \eqref{eq:cov-ai} and \eqref{eq:cov-bi} in Eq. \eqref{eq:cov-big2} results in an equation for the covariance of two Classify and Count quantifiers at different thresholds $x$ and $y$:
    
    \begin{align}
        C[ \hat\alpha_\text{CC}(x, D_\text{test}), \hat\alpha_\text{CC}(y, D_\text{test}) ] &= \frac{n_\text{test}^+}{n_\text{test}^2} F^+(y)(1-F^+(x)) + \frac{n_\text{test}^-}{n_\text{test}^2} F^-(y)(1-F^-(x)) \nonumber \\
        &= \frac{\alpha}{n_\text{test}} F^+(y)(1-F^+(x)) + \frac{1 - \alpha}{n_\text{test}} F^-(y)(1-F^-(x)). \label{eq:cov-two-cc}
    \end{align}  
    Finally, substituting Eq. \eqref{eq:cov-two-cc} in Eq. \eqref{eq:cov-cs-global} results in the variance of Continuous Sweep as
     \begin{align}
        V \left[ \hat\alpha_\text{CS} \right] &= \frac{2}{n_\text{test}(\theta_\text{r} - \theta_\text{l})^2} \int_{y = \theta_\text{l}}^{\theta_\text{r}} \int_{x = y}^{\theta_\text{r}}\frac{ \alpha F^+(y) (1 - F^+(x)) + (1 - \alpha) F^-(y) (1 - F^-(x))}{(F^+(x) - F^-(x)) (F^+(y) - F^-(y))} dxdy. \label{eq:var-cs2}
    \end{align}
    This concludes the proof of the expressions for the variance of the Continuous Sweep. 
    \end{proof}
\end{theorem}

Now that we have derived the bias and variance of the Continuous Sweep quantifier, we will show how to use these results to derive an optimal value for the threshold $p^\Delta$.

\subsection{Optimal threshold $p^\Delta$} \label{subsec:pdelta}
\noindent The analytic expression for the variance of Continuous Sweep shows that its value is dependent on decision boundaries $\theta_\text{l}$ and $\theta_\text{r}$. Those two values are determined by $p^\Delta$, the minimum difference between the true and false positive rate of the data points to be included. As shown in Figure \ref{fig:pdelta_var}, a large value of $p^\Delta$ results in a small difference between $\theta_\text{l}$ and $\theta_\text{r}$, whereas a small value of $p^\Delta$ results in a large difference between $\theta_\text{l}$ and $\theta_\text{r}$. 

A larger difference between the decision boundaries includes more, but potentially unreliable, Adjusted Count estimates in the prevalence estimate. This trade-off is illustrated in Figure \ref{fig:pdelta_var}, where an extremely low $p^\Delta$ leads to extreme variances and the maximum possible value of $p^\Delta$ leads to the variance of only one adjusted count estimate. Forman \cite{forman2005} proposed for Median Sweep a threshold equal to $p^\Delta = \frac{1}{4}$, which we will call the \emph{traditional} threshold. For Continuous Sweep, we choose an optimized value for $p^\Delta$ using the derivations of Section \ref{subsec:proofs}. Minimizing Eq. \eqref{eq:var-cs}, viewed as a function of $p^\Delta$, returns the \textit{optimal} threshold $p^\Delta$, in the sense that it yields a prevalence estimate with minimal variance (and hence the minimal MSE due to Theorem \ref{thm:ev-ac}). Optimizing $p^\Delta$ improves the heuristic decision rule of Median Sweep and is therefore the third and final property of Median Sweep that has been changed in Continuous Sweep.

In order to use the analytic expressions in the optimization of $p^\Delta$, we are obliged to make two practical choices. First, the variance function (Eq. \eqref{eq:var-cs}) has parameters that are often unknown in practice. For example, the prevalence $\alpha$ is a parameter in Eq. \eqref{eq:var-cs}. Because the prevalence is unknown, one can use $\alpha = 0.5$ as a plug-in estimate. We experimented that the prevalence has a negligible impact on the total variance and hence on the optimal $p^\Delta$.  Therefore, the plug-in estimator of $\alpha = 0.5$ can reliably be used in Eq. \eqref{eq:var-cs} and we do not include to maintain focus on this parameter. Second, the parameters of $F^+$ and $F^-$ may be estimated using $D_\text{train}$ if they are unknown. In that case, the optimal threshold can be estimated by using plug-in estimators for the parameters of $F^+$ and $F^-$ in Eq. \eqref{eq:var-cs}. In our simulation studies, we use the \verb|optim| function in R to approximate the optimal threshold numerically in this manner. 

\begin{figure}
    \centering
    \begin{subfigure}{0.49\textwidth}
        \centering
        \includegraphics[width = \textwidth]{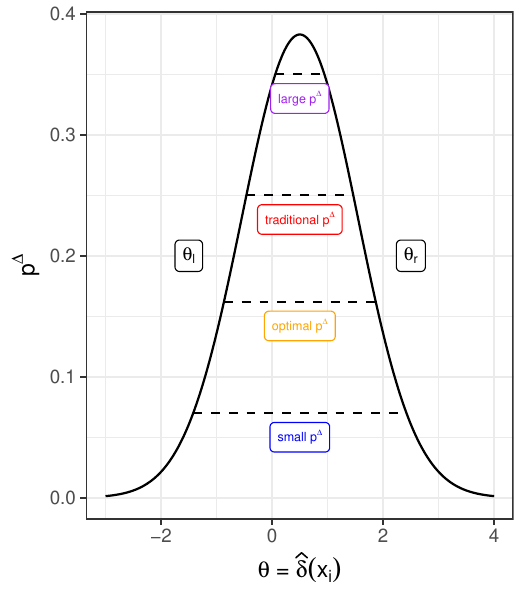}
        \caption{}
        \label{fig:pdelta_var}
    \end{subfigure}
    \hfill
    \begin{subfigure}{0.49\textwidth}
        \centering
        \includegraphics[width = \textwidth]{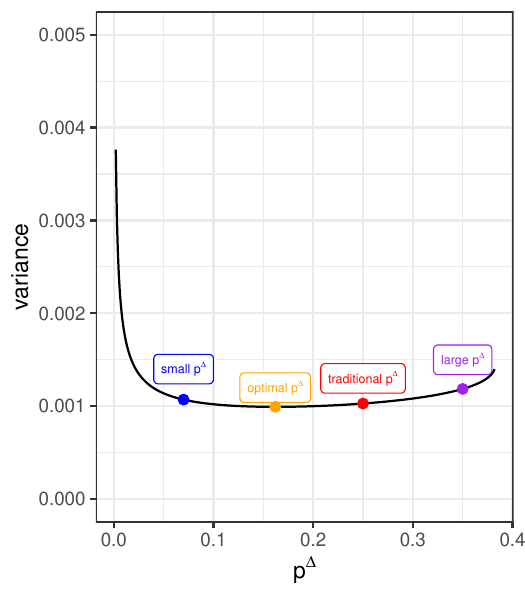}
        \caption{}
        \label{fig:pdelta_range}
    \end{subfigure}
    \caption{Panel (a) illustrates the difference between $F^+(\theta)$ and $F^-(\theta)$ for each $\theta$. Panel (b) illustrates the theoretical variance at different $p^\Delta$ of the Continuous Sweep with a $D_\text{test}$ of size $1000$ and prevalence $\alpha = 0.5$, where $F^+ \sim N(1, 1)$ and $F^- \sim N(0, 1)$.}
    \label{fig:pdelta}
\end{figure}

\section{Simulation Study}
\noindent In this section, we conduct three simulation studies to evaluate the performance of Continuous Sweep in comparison to other quantifiers. 
In the first study, we assume that both $F^+$ and $F^-$ follow a normal distribution with known parameters. This study allows us to compare the quantifiers under ideal conditions, where the assumptions align with those of Theorem \ref{thm:ev-ac} and \ref{thm:var-ac}. In real-world scenarios, $F^+$ and $F^-$ must be estimated from training data. So, in the second study, we still assume that $F^+$ and $F^-$ follow a normal distribution, but their parameters are now estimated from the data, resulting in $\hat F^+$ and $\hat F^-$. The third study examines a case where the data follows a skew-normal distribution. For the third study, we fit two versions of Continuous Sweep, one assuming normal data and the other assuming skew-normal distributed data.

For each study, we divide the comparison of the quantifiers into two groups. First, we compare Continuous Sweep with other quantifiers within the \emph{Classify, Count, and Correct} group. Second, we compare Continuous Sweep with two quantifiers from the \emph{Distribution Matching} group, namely DyS and SLD, which had strong performances at the LeQua2022 competition.      

\subsection{Study 1: Known parameters for $F^+$ and $F^-$ using a normal distribution}

\subsubsection*{Outline and design}
The first simulation study assumes that the discriminant scores of the positive class follow a normal distribution $F^+$ with mean $\mu^+ = 1$ and a standard deviation $\sigma^{+}$, while the scores of the negative class follow a normal distribution $F^-$ with mean $\mu^- = 0$ and a standard deviation of $\sigma^{-}$. Moreover, in this study we assume that the mean and standard deviation of classes are known. The test data $D_\text{test}$ contain a proportion $\alpha$ (prevalence) of discriminant scores sampled from $F^+$ and a proportion $1 - \alpha$ of discriminant scores sampled from $F^-$. 

Using the known distributions of the discriminant scores as "training data" and a sampled test set $D_\text{test}$, Continuous Sweep can estimate the prevalence with Algorithm \ref{alg:cont-sweep}. Estimating the prevalence using Median Sweep is described in Section $2$. The decision boundaries $\theta_l$ and $\theta_r$ are determined using the known distributions $F^+$ and $F^-$. Consequently, we also use $F^+$ and $F^-$ in Eq. \eqref{eq:part-cont-sweep} instead of $\hat F^+$ and $\hat F^-$, because we assume that the marginal distributions are known. To compute Median Sweep, $F^+$ and $F^-$ replace $p^+$ and $p^-$ in both Eq. \eqref{eq:ac-theta} and the decision rule, in order to obtain $\boldsymbol{\theta}_\text{MS}$.

As described in Section \ref{subsec:pdelta}, the performance of Continuous Sweep, but also of Median Sweep, is dependent on $p^\Delta$. Hence, we apply the traditional value of $p^\Delta = \frac{1}{4}$ and the optimal $p^\Delta$ on both quantifiers. As we are unable to obtain an optimal $p^\Delta$ using Median Sweep, we used the optimal $p^\Delta$ estimated by Continuous Sweep. Consequently, this results in a comparison between four quantifiers in the group \emph{Classify, Count, and Correct}: 1) Optimal Continuous Sweep (O-CS), 2) Traditional Continuous Sweep (T-CS), 3) Optimal Median Sweep (O-MS), and 4) Traditional Median Sweep (T-MS). 

The two distribution matchers, SLD and DyS, use probabilities instead of discriminant scores. Therefore, we transform the discriminant scores to probabilities by using Naive Bayes with a prevalence of $0.5$ and the known distributions $F^+$ and $F^-$. Moreover, SLD needs the prevalence of the training set as input of their algorithm. In line with the transformation of the discriminant scores, we set this prevalence to $0.5$. DyS needs as input the histogram of the training data of the probabilities in the positive and negative class. Since the transformation from discriminant scores to probabilities is non-inversible, we cannot derive the distribution of the probabilities directly. Therefore, we sample this distribution with many samples in both the positive and negative class ($n = 10^7$) to accurately estimate these "known" distributions and make the comparison between the quantifiers fair.

The design of the simulation study is based on a full-factorial design of four factors. The number of observations $n_\text{test}$ in $D_\text{test}$ is either $100$ or $1000$, which is mainly intended to investigate the relationship between sample size and variance. To investigate the relationship between the classifier and the quantifiers, the possible values of $\sigma^+$ and $\sigma^-$ are both set to $0.5$, $1$, and $1.5$. Last, the prevalence $\alpha_\text{test}$ can take values of $0.3$, $0.5$, and $0.9$ to examine whether the quantifiers behave differently under either balanced or imbalanced data. The values of $\mu^+ = 1$ and $\mu^- = 0$ are fixed across the simulation study. In summary, we designed the simulation study using $2 \times 3 \times 3 \times 3 = 54$ different situations. For each situation, we generate $10,\!000$ test sets $D_\text{test}$ used to compute the prevalence estimates from the six quantifiers. With the prevalence estimates, we can estimate the bias, variance and mean squared error (MSE) for the six quantifiers under each of the $54$ situations. 

\begin{figure}
    \centering
    \includegraphics[width = \textwidth]{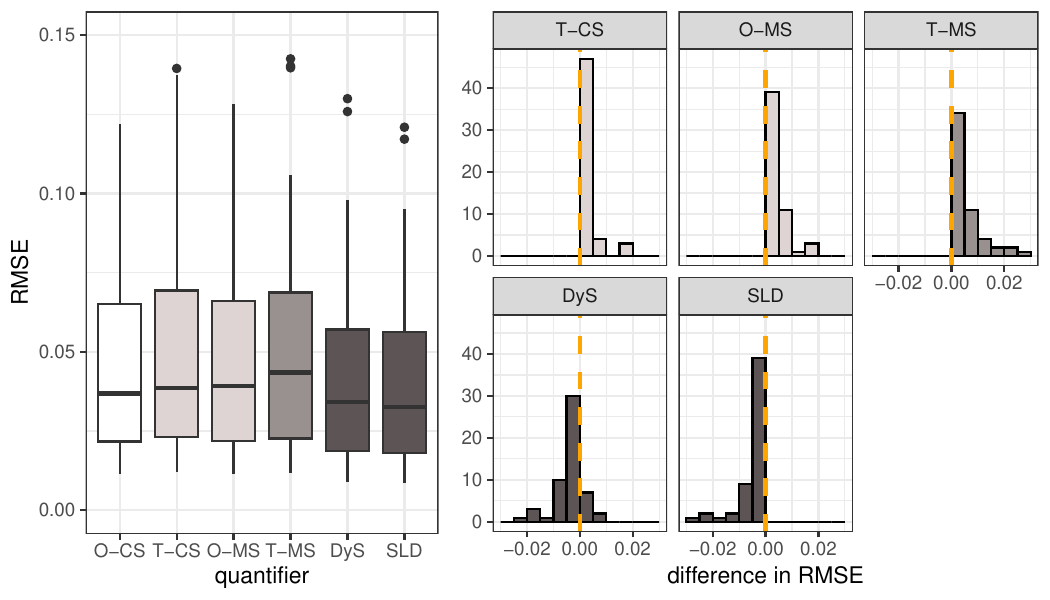}
    \caption{Simulation study 1: Comparison of the quantifiers against optimal Continuous Sweep (O-CS) using boxplots and histograms. The boxplots (left) show the distribution of the RMSE across the 54 situations of the six quantifiers. The histograms (right) show the difference of the RMSE between the other quantifiers against O-CS. A negative value indicates that O-CS has a higher RMSE than the other quantifier, vice versa for positive values.}
    \label{fig:s1-box-hist}
\end{figure}

\subsubsection*{Results}
Figure \ref{fig:s1-box-hist} compares the quantifiers in terms of RMSE. For the quantifiers in the group \emph{Classify, Count and Correct}, O-CS has the lowest RMSE of the four quantifiers in each of the $54$ situations and always outperforms the other quantifiers. Moreover, T-CS is often second best, but is outperformed by one or both Median Sweep quantifiers in some situations. The difference in RMSE between O-CS and the other quantifiers is often small, but is in some situations larger, especially compared with T-MS. For the quantifiers in the group \emph{Distribution Matching}, we see that SLD outperforms all quantifiers in each situation. The difference between SLD and O-CS is comparable with the difference between O-CS and T-MS. O-CS outperforms DyS in $9$ of the $54$ situations. 


\subsection{Study 2: Estimated parameters for $F^+$ and $F^-$ using a normal distribution}

\subsubsection*{Outline and design}
In the second simulation study, we make the assumption that the parameters of $F^+$ and $F^-$ are unknown and have to be estimated. Like the first simulation study, the positive discriminant scores are normally distributed with mean $\mu^+ = 1$, and standard deviation $\sigma^+$, and the negative discriminant scores are normally distributed with mean $\mu^- = 0$, and standard deviation $\sigma^-$. The parameters of $\hat F^+$ and $\hat F^-$ can be estimated by maximum likelihood estimation. With these estimated parameters, we can derive the decision boundaries $\theta_l$ and $\theta_r$ using $\hat F^+$ and $\hat F^-$. Again, we can use Algorithm \ref{alg:cont-sweep} to compute the prevalence estimates of Continuous Sweep. Median Sweep estimates the true and false positive rate using empirical cumulative density functions $p^+$ and $p^-$.  

The discriminant scores in $D_\text{train}$ are converted to probabilities by using Naive Bayes with a prevalence of $\alpha_{train}$ and the known distributions $F^+$ and $F^-$. The converted discriminant scores in $D_\text{test}$ also uses Naive Bayes with a prevalence of $\alpha_{train}$ and the known distributions $F^+$ and $F^-$ to match the probabilities in $D_\text{train}$. In contrast to the first simulation study, the parameter $\alpha_\text{train}$ is available as input parameter for SLD. DyS needs as input the histogram of the training data of the probabilities in the positive and negative class, which we can construct from the transformed discriminant scores. Converting from discriminant scores to probabilities in this way may be slightly in favor for the distribution matchers, but is the most consistent in the set-up of this simulation study.

In the second simulation study, we use the same $54$ settings for $n_\text{test}$, $\alpha_\text{test}$, $\mu^+$, $\mu^-$, $\sigma^+$, and $\sigma^-$. In addition, we add two extra variables $n_\text{train}$ and $\alpha_\text{train}$. We vary the size $n_\text{train}$ of the training set between $100$ and $1000$ and its prevalence $\alpha_\text{train}$ between $0.4$, $0.5$, and $0.6$. The reason why we vary little in $\alpha_\text{train}$ between training sets, is that extreme prevalence values return only a few data points for one class for small training sets, but we still want to investigate whether $\alpha_\text{train}$ has a large effect on the bias, variance and MSE of the quantifiers. The number of situations is therefore increased to $54 \times 2 \times 3 = 324$ situations. 

\begin{figure}
    \centering
    \includegraphics[width = \textwidth]{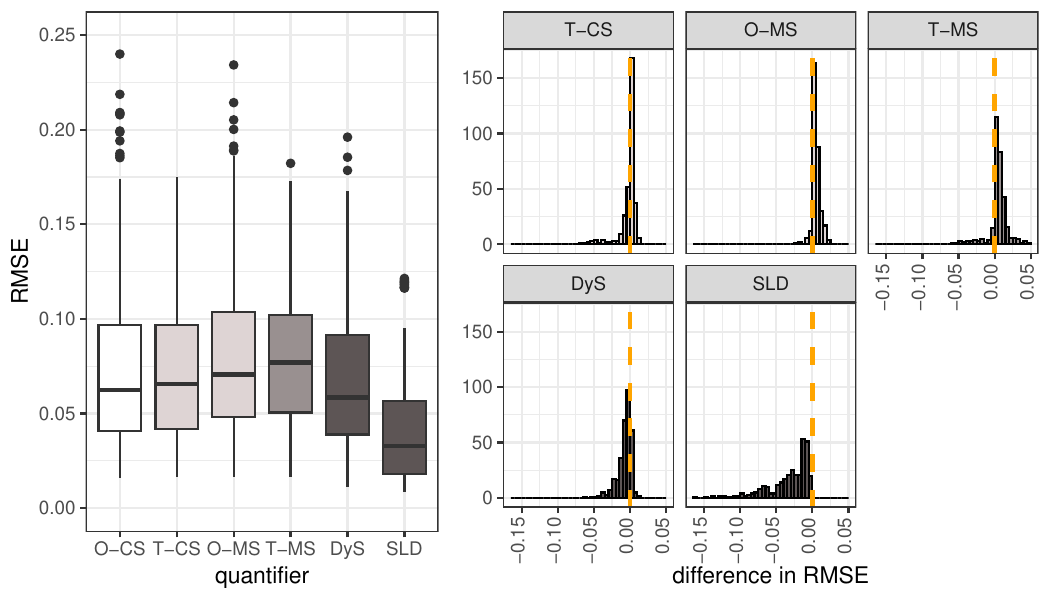}
    \caption{Simulation study 2: Comparison of the quantifiers against optimal Continuous Sweep (O-CS) using boxplots and histograms. The boxplots (left) show the distribution of the RMSE across the $324$ situations of the six quantifiers. The histograms (right) show the difference of the RMSE between the other quantifiers against O-CS. A negative value indicates that O-CS has a higher RMSE than the other quantifier, vice versa for positive values. The values for the RMSE in study 2 are larger than the values in study 1, because there is more uncertainty in estimating the true and false positive rates.}
    \label{fig:s2-box-hist}
\end{figure}

\subsubsection*{Results}
Figure \ref{fig:s2-box-hist} compares the quantifiers in terms of RMSE. For the quantifiers in the group \emph{Classify, Count and Correct}, O-CS has the lowest RMSE of the four quantifiers across the $324$ conditions and outperforms the other quantifiers in a large majority ($211$ of out $324$) of the situations. In the other situations, T-CS outperforms O-CS, T-MS, and O-MS. This quantifier seems more robust for extreme cases, for example the ones where the $p^\Delta$ cannot be estimated accurately. Specifically, O-CS performs the best when the class majority belonged to a class with small variance, when the training set was large, and when the test set was small. In these cases, the optimal $p^\Delta$ can be estimated most accurately. The order of the performance between the quantifiers remains the same in the second simulation study, with O-CS performing the best and T-MS performing the worst out of the four \emph{Classify, Count and Correct} quantifiers. For the quantifiers in the group \emph{Distribution Matching}, we see that SLD outperforms all quantifiers in each situation as in Study 1. SLD performs relatively better in the second simulation study than the first, but is slightly in favor due to the characteristics of the study's design. In $68$ of out $324$ situations, O-CS outperforms DyS.

\subsection{Study 3: Misfit on more complex distributions}

\subsubsection*{Outline and design}
Similar to the second simulation study, the third simulation study assumes that the parameters of $F^+$ and $F^-$ have to be estimated from the training data. In study 2, we assumed that we correctly specified that the training and test data follow a normal distribution. In study 3, the training and test data follow a skew-normal distribution and we fit two types of Continuous Sweep: one that assumes a skew-normal distribution (correct fit) and one that assumes a normal distribution (incorrect fit). We fit Median Sweep and Continuous Sweep using the normal distribution in the same way as in the second simulation study. Continuous Sweep using the skew-normal distribution works similarly, but an additional shape parameter is required. We estimate the shape parameter using maximum likelihood estimation, since no analytic solution exists. 

Fitting Continuous Sweep with a skew-normal distribution takes substantially more time than fitting Continuous Sweep with a normal distribution. Therefore, we take a subset of the design given in the second simulation study. For all simulation studies, $D_\text{train}$ consists of $1000$ observations with prevalence $\alpha_\text{train} = 0.5$. The test set $D_\text{test}$ can vary between $100$ and $1000$ observations, the standard deviations $\sigma^+$ and $\sigma^-$ can both take the values $0.5$ and $1.5$, and the test set prevalence $\alpha_\text{test}$ can vary between $0.3$, $0.5$, and $0.9$. This results in $2 \times 2 \times 2 \times 3 = 24$ situations. Each of the 24 situations are fitted on three small simulation studies: the first study considers little skew in the data (skew = 1), the second study considers medium skew (skew = 2) and the third study considers large skew (skew = 4). The set-up for SLD and DyS is similar, but we adjust the marginal density from a normal to a skew-normal distribution. 

To summarize, we compare within the group \emph{Classify, Count, and Correct} six quantifiers: Median Sweep (MS), Continuous Sweep with normal fit (CS (norm)), and Continuous Sweep with skew-normal fit (CS (skew)), all with an optimal (O) and a traditional (T) $p^\Delta$. Within the group of \emph{Distribution Matchers}, we compare O-CS (skew) with SLD and DyS. The skewness of the data is either small, medium, or large, resulting in three small simulation studies on $24$ different situations of the data.

\begin{figure}
    \centering
    \begin{subfigure}[b]{\textwidth}
        \includegraphics[width=\textwidth]{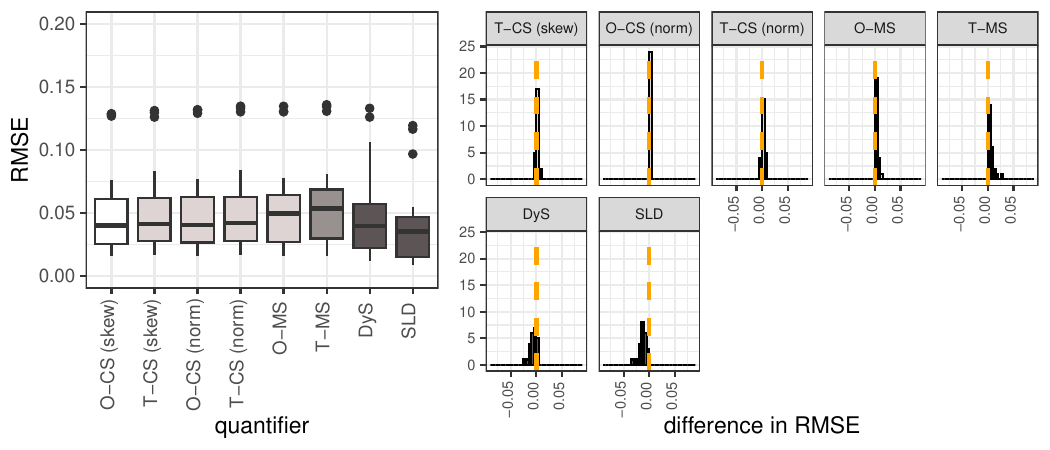}
		\caption{skew = 1}
        \label{fig:s3a-box-hist}
    \end{subfigure}
    
    \begin{subfigure}[b]{\textwidth}
        \includegraphics[width=\textwidth]{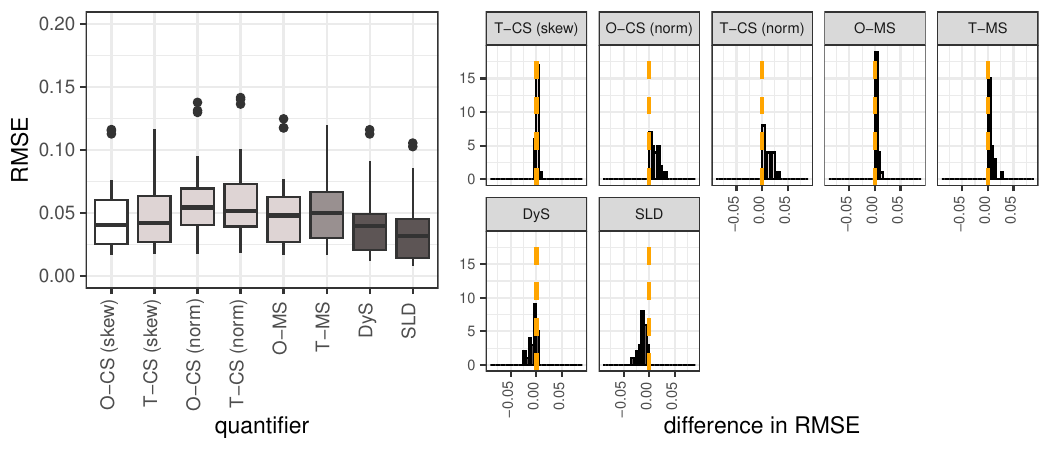}
		\caption{skew = 2}
        \label{fig:s3b-box-hist}
    \end{subfigure}
    
    \begin{subfigure}[b]{\textwidth}
        \includegraphics[width=\textwidth]{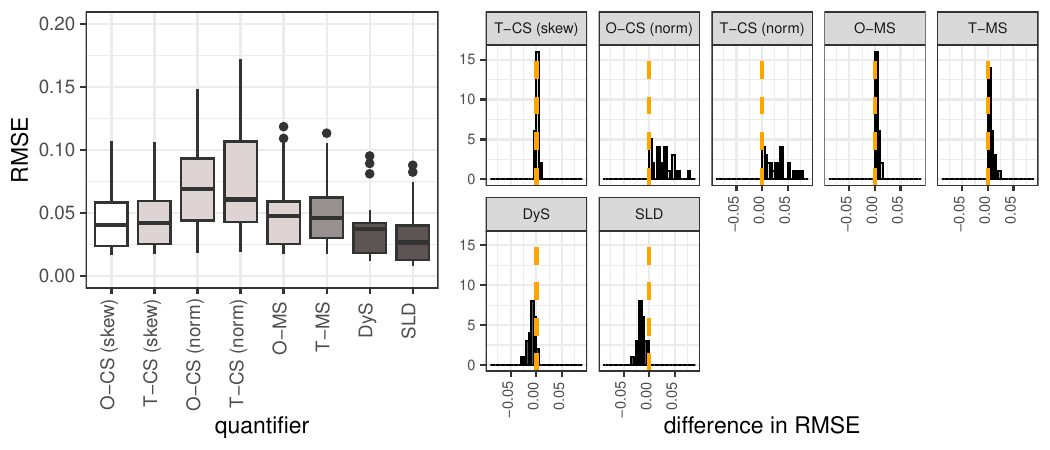}
		\caption{skew = 4}
        \label{fig:s3c-box-hist}
    \end{subfigure}
    
    \caption{Simulation study 3: Comparison of the quantifiers against optimal Continuous Sweep (O-CS (skew)) using boxplots and histograms. The boxplots (left) show the distribution of the RMSE across the 54 situations of the six quantifiers. The histograms (right) show the difference of the RMSE between the other quantifiers against O-CS (skew). A negative value indicates that O-CS (skew) has a higher RMSE than the other quantifier, vice versa for positive values.}
    \label{fig:s3-box-hist}
\end{figure}

\subsubsection*{Results}
Figure \ref{fig:s3-box-hist} compares the quantifiers in terms of RMSE. In general, the RMSE is lower for higher values of skew. For each level of skewness, the Continuous Sweep quantifiers with skew-normal fit consistently outperform the other quantifiers in the group \emph{Classify, Count, and Correct}. O-CS (norm) outperforms MS if the skew is small, but performs worse when the skew increases. In general, an optimal $p^\Delta$ performs better across the situations than a traditional $p^\Delta = \frac{1}{4}$ across all quantifiers and levels of skewness. O-CS (skew) always outperforms O-CS (norm) in all conditions for all skew levels. For the group of \emph{Distribution Matching}, we see a similar pattern between the different skew levels, where SLD always outperforms the other quantifiers and DyS is sometimes outperformed by O-CS (skew).

\subsection{Summary} \label{subsec:summary}
We compared Continuous Sweep and Median Sweep by three simulation studies. The first simulation study showed that Continuous Sweep with an optimal threshold always outperforms Median Sweep assuming that the distributions for the positive and negative discriminant scores are normal with known parameters. The second simulation study showed that either Continuous Sweep with an optimal threshold or Continuous Sweep with the traditional threshold outperforms Median Sweep assuming that the distributions for the positive and negative discriminant scores are normal with unknown parameters. The third simulation study shows that Continuous Sweep outperforms Median Sweep if the distributions of the discriminant scores are correctly specified (skew-normal data with skew-normal model assumptions), but underperformed when the distributions of the discriminant scores are misspecified (skew-normal data with normal model assumptions). SLD outperformed all quantifiers and DyS often outperformed Continuous Sweep is this simulation design. 

\section{Real-world application: LeQua2022 competition}

In addition to a simulation study, we compared Continuous Sweep to other quantifiers in terms of real data. This data originated from LeQua2022: a competition designed for quantifiers. We applied Continuous Sweep to the data and compared the results with Median Sweep and two other quantifiers: SLD and DyS. Both SLD and DyS performed splendid in LeQua2022 and were therefore included in our comparison between quantifiers. In this section, we describe the data, demonstrate how to fit Continuous Sweep, and analyze the results.

\subsection{Data description}
Each observation in the LeQua2022 data contained $300$ vector representations of textual product reviews as predictor variables and one label indicating whether the review was either positively or negatively rated. The data is split in two parts: 1) one training set with $5000$ observations where the label of each observation is known, and 2) $5000$ test sets with $250$ observations where, in contrast to the training set, the labels are unknown. The quantifiers aim to estimate the prevalence of all test sets as well as possible. The prevalence across the test sets varied between zero and one, aiming that the quantifiers should be able to make good estimates for a wide range of prevalence values. The quantifiers are evaluated on the mean absolute error (MAE) between the true and estimated prevalence across the $5000$ test sets.

\subsection{Fitting the quantifier and underlying classifier}
The training data has been used to fit the quantifier and the underlying classifier. Using Tidymodels \cite{tidymodels2020}, we chose a support vector machine (SVM) with a radial basis function, following the top-competitor of the LeQua2022 competition. We tuned the SVM using grid search choosing the hyperparameters that maximized the ROC-AUC metric, which resulted in values $C = 21.485$ and $\sigma = 1.1178 \times 10^{-4}$.

After we fitted the underlying classifier, we applied our quantifier Continuous Sweep. The first step in fitting Continuous Sweep, is obtaining the observations' scores from the training data and these were computed using 5-fold cross validation. We obtained two types of scores using Platt-scaling: 1) probabilities that were used for SLD and DyS, and 2) the raw distances that were used for Continuous Sweep and Median Sweep.

In order to fit Continuous Sweep, we separated the scores for the positive and the negative labels. We chose to fit the densities with a skew-normal distribution for both the positive and negative scores. The parameters that have to be estimated are the shape ($\xi$), scale ($\omega)$, and skew ($\beta$). We estimated that the positive scores were skew-normally distributed parameter values $\xi = 1.0309, \omega = 1.4119, \beta = 1.0068$, and the negatives with values $\xi = 0.8425, \omega = 1.6677, \beta = -1.7983$ respectively. We optimized the threshold resulting in $p^\Delta = 0.400$. 

Median Sweep is fitted similarly as Continuous Sweep. We separated the scores for the positive and the negative labels, but used an empirical cumulative density function (ecdf) to fit the marginal distributions. The threshold is the default from Forman \cite{forman2008}, with $p^\Delta = 0.25$. DyS and SLD are fitted using the default values from the QuaPy Python package alongside the tuned SVM. DyS has $8$ bins with a tolerance value of $10^{-5}$ and SLD has a maximum of $1000$ iterations with a tolerance value of $10^{-4}$. 

\subsection{Results}
We compared the performance of Continuous Sweep (O-CS) with Median Sweep (T-MS), SLD, and DyS using $5000$ test sets. After we have computed $5000$ prevalence estimates, we summarized the results in Table \ref{tbl:lq-mae} and Figure \ref{fig:lq-whole}. In terms of RMSE, MAE, and RAE, O-CS performed much better than T-MS, but slightly worse than SLD and DyS, which can be confirmed by the boxplots in Figure \ref{fig:lq-whole}. The performance of SLD and DyS is slightly different between the results published from the LeQua2022 competition and our application. The MAE of SLD and DyS are in the competition $0.0252$ and $0.0281$ respectively, against $0.0231$ and $0.0251$ in our application. The main reason of this difference, is the difference in the underlying classifier. 

From the boxplots, it looks like that SLD and DyS had less extreme outliers than O-CS and T-MS. Moreover, it appears that DyS has a small bias, especially compared with the other quantifiers. In Figure \ref{fig:lq-sub3}, we see that O-CS outperforms T-MS in $2957$ test sets, outperforms DyS in $2508$ test sets, and outperforms SLD in $2188$ test sets. In Figure \ref{fig:lq-sub4}, we see that the performance of all quantifiers is quite similar around a balanced prevalence of $\alpha = 0.5$. Given the true prevalence, the variance between the differences in AE between O-CS and T-MS, and between O-CS and DyS, seem constant across each prevalence value. For extreme prevalence values, O-CS performs better than T-MS and slightly worse than DyS. The difference between O-CS and SLD is not constant given the true prevalence values. For prevalence values around $\alpha = 0.5$, the predictions of O-CS and SLD are much closer to each other than for extreme prevalence values. SLD performs better than O-CS for extreme prevalence values. 

\begin{table}[ht]
\centering
\caption{Error metrics of $5000$ test sets from the LeQua 2022 competition for four quantifiers.}
\begin{tabular}{lrrr}
\toprule
\textbf{Quantifier} & \textbf{MAE} & \textbf{RMSE} & \textbf{RAE} \\
\midrule
SLD & 0.0231 & 0.0292 & 0.1103 \\
DyS & 0.0251 & 0.0316 & 0.1314 \\
O-CS & 0.0256 & 0.0322 & 0.1556 \\
T-MS & 0.0294 & 0.0375 & 0.2123 \\
\bottomrule
\end{tabular}
\label{tbl:lq-mae}
\end{table}

\begin{figure}
    \centering
    \begin{subfigure}{0.45\textwidth}
        \centering
        \includegraphics[width=\textwidth]{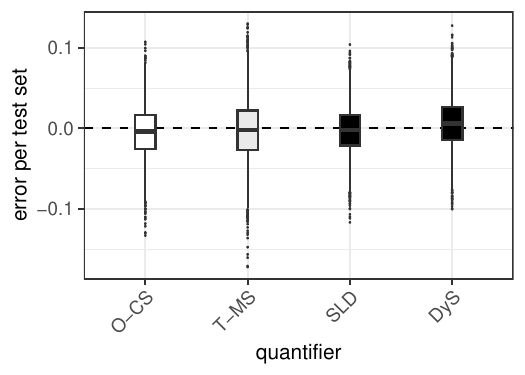}
        \caption{Difference between predicted prevalence and true prevalence.}
        \label{fig:lq-sub1}
    \end{subfigure}
    \hfill
    \begin{subfigure}{0.45\textwidth}
        \centering
        \includegraphics[width=\textwidth]{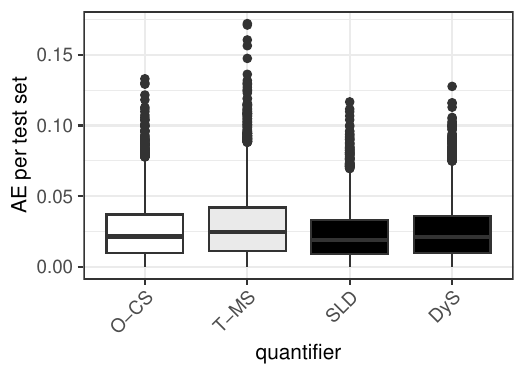}
        \caption{Absolute difference of predicted prevalence and true prevalence.}
        \label{fig:lq-sub2}
    \end{subfigure}
    \vskip\baselineskip
    \begin{subfigure}{0.45\textwidth}
        \centering
        \includegraphics[width=\textwidth]{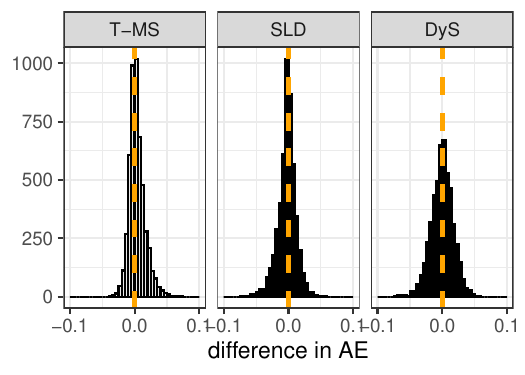}
        \caption{Histogram of difference in AE between quantifier and O-CS.}
        \label{fig:lq-sub3}
    \end{subfigure}
    \hfill
    \begin{subfigure}{0.45\textwidth}
        \centering
        \includegraphics[width=\textwidth]{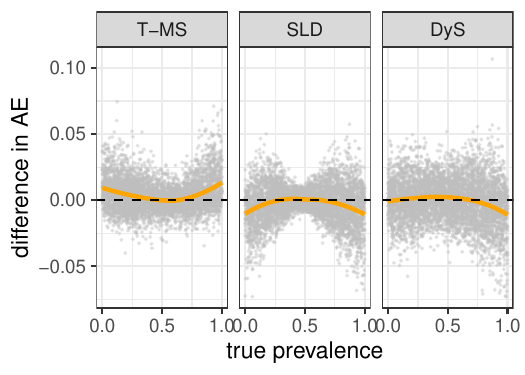}
        \caption{Scatterplot of difference in AE between quantifier and O-CS.}
        \label{fig:lq-sub4}
    \end{subfigure}
    \caption{Comparing the outcomes of the $5,000$ test sets from the LeQua2022 competition. In Figure (a) and (b), the error and absolute error (estimate - true value) against each quantifier are compared. In Figure (c), the distribution of the difference in AE between three quantifiers and O-CS are compared, where a positive value indicates that O-CS has a lower AE value. In Figure (d), a similar plot is presented, but with the addition of the true prevalence on the x-axis.}
    \label{fig:lq-whole}
\end{figure}

\subsection{Summary Simulation Study}
We compared optimal Continuous Sweep (O-CS) with Median Sweep (T-MS), SLD and DyS using the LeQua2022 competition data. Across $5000$ test sets with prevalences varying between $0$ and $1$, we found that O-CS performed on average better than T-MS, slightly worse than DyS and is outperformed by SLD. The prevalence estimates of O-CS are close to the prevalence estimates of SLD for prevalence values around $\alpha = 0.5$. 

\newpage
\section{Discussion}
\noindent In this paper, we studied the group of \emph{Classify, Count, and Correct} quantifiers. In Section 2, we reviewed the development of quantifiers in this group. In Section 3, we took the most advanced and best performing quantifier from this group, Median Sweep, and modified it in order to enable a theoretical analysis. We named the resulting quantifier Continuous Sweep. Continuous Sweep modifies Median Sweep in two aspects: 1) instead of the empirical distributions of the true and false positive rate, we used a theoretical distribution; 2) instead of the median of a set of adjusted count estimates we used the mean. These modifications enabled us to derive analytical expressions for the bias and variance of Continuous Sweep under a common set of assumptions. As such Continuous Sweep is one of a few quantifiers for which analytical results are available. We used the expressions for bias and variance to find an optimal value of the threshold hyperparameter ($p^\Delta$), a value that minimizes the mean squared error and therefore provides a better quantifier. 


In addition to deriving these theoretical results, we compared the performance of Continous Sweep to other quantifiers. 
In Section $4$ we presented three simulation studies. In the first study, the distributions of the true and positive rate are known. In this case, the results show that Continuous Sweep is always better than Median Sweep. In the second and third simulation study, the distributions of the true and positive rate need to be estimated from the training data. In simulation study 2, the rates are estimated using a correctly specified distribution. Moreover, in simulation study 3, we also use a misspecified distribution to estimate the rates. We demonstrated that Continuous Sweep has a lower mean squared error than Median Sweep in most situations, except when the estimated distribution is misspecified (in simulation study 3). Overall, these results present a substantial contribution in the search of better quantifiers in the group \emph{Classify, Count, and Correct}.

Furthermore, in the three simulation studies we also compared Continuous Sweep with two quantifiers in the group \emph{Distribution Matching}, the DyS and SLD quantifiers. The quantifier SLD did generally well and (slightly) outperformed Continuous Sweep. Continuous Sweep turned out to be competitive to DyS. 

In Section $5$, we used the LeQua2022 competition data to compare the performance of Continuous Sweep with other quantifiers. In the data, one training data set is available and $5000$ test sets. After the competition the true prevalences of the test sets were made public. Therefore, we can compare the performance of the quantifiers on these $5000$ test sets. On average, Continuous Sweep outperformed Median Sweep, performed similar as the DyS quantifier, but performed worse than SLD. However, checking the performance on individual test sets, we see that Continuous Sweep performed better than SLD on $2188$ test sets and better than DyS on $2508$ test sets (cf. Figure \ref{fig:lq-whole}).

To summarize, we showed that Continuous Sweep generally outperforms Median Sweep and is competitive to SLD en DyS when applied to the LeQua2022 competition data. \cite{schumacher2021} showed that Median Sweep performs very well among many other datasets and, in some of them, outperforms DyS and SLD. Combining the results from \cite{schumacher2021} with the results of our paper could imply that Continuous Sweep outperforms DyS and SLD in these datasets as well. Future work could focus on investigating this implication and make more general statements about the comparison between Continuous Sweep and SLD and DyS. 

As a side note, we like to remark that the comparative study in \cite{esuli2023learning} showed that Median Sweep underperforms compared to most quantifiers. However, that study used an older version of Python package QuaPy \cite{moreo2021quapy} with a bug in the code for Median Sweep. In the most recent version of the QuaPy package \cite{moreo2021quapy}, the bug has been fixed. The comparative study needs to be updated to provide new results \cite{esuli2023learning}.

As future work, we identify three main directions to further investigate, based on either the limiting assumptions that we have made throughout the paper or on the limiting construction of the studies. First, we assumed in our derivations of the bias and variance of Continuous Sweep, the distributions of the true and false positive rates to be known. In practice, both rates have to be estimated using the training data, which adds an extra source of variance to the computations. In most applications, we have to take this extra source of variance into account and therefore we need to extend the derivations. In future research, the derivations could be extended by loosening this assumption such that the parameters of the distributions are being estimated using training data. Therefore the bias and variance of the estimated true and false positive rates need to be incorporated. Subsequently, the optimal threshold can be adapted using such new expressions. 

A second direction regards an obstacle in comparing quantifiers for which many quantifiers use the input of a classifier. Instead of simply evaluating a quantifier, a classifier-quantifier pair is evaluated. Moreover, optimizing a classifier for one quantifier could work differently than optimizing a classifier for a different quantifier. Similarly, one classifier might be optimal for quantifier A, while another classifier is optimal for quantifier B. This combination of classifiers and quantifiers makes an honest comparison of quantifiers difficult. The direct quantification learners do not need to train an underlying classifier. This could in principle be an advantage. Although, comparative studies do not show that these quantifiers outperform others \cite{schumacher2021}. 

The third and final limiting assumption in this paper is that Continuous Sweep uses parametric distributions to estimate the true and false positive rate. For example, in the simulation study, normal and skew-normal distributions have been used to estimate the rates, as well as a skew-normal distribution in the LeQua2022 application. We like to point out that any parametric distribution can be used to estimate these rates. In applications it is important to find an accurate distribution, as misspecification leads to decreased performance. If accurate distributions cannot be obtained using parametric distributions, continuous non-parametric distributions obtained from kernel density estimators can be used to estimate the true and false positive rates. 

In conclusion, we developed a new quantifier called Continuous Sweep that excels in the group \emph{Classify, Count, and Correct}. Furthermore, we derived analytical expressions of its bias and variance. These expressions contribute to the research gap of performing more theoretical analysis in quantification learning \cite{gonzalez2017quantification}. Based on the expressions, we were able to derive an optimal value for the threshold selection policy. The optimal value minimizes the variance given the distribution functions of the true and false positive rates. The resulting optimal Continuous Sweep outperforms Median Sweep and is competitive with state-of-the-art quantifiers DyS and SLD on real-world data. We believe that Continuous Sweep is a quantifier with great potential for which further comparisons with existing quantifiers are needed and extensions to more general settings are possible.

\section*{Acknowledgements}
The research of the first author is supported by LUXs Data Science B.V., Leiden.

\newpage
\bibliography{references}


\begin{thebibliography}{33}
\ifx \bisbn   \undefined \def \bisbn  #1{ISBN #1}\fi
\ifx \binits  \undefined \def \binits#1{#1}\fi
\ifx \bauthor  \undefined \def \bauthor#1{#1}\fi
\ifx \batitle  \undefined \def \batitle#1{#1}\fi
\ifx \bjtitle  \undefined \def \bjtitle#1{#1}\fi
\ifx \bvolume  \undefined \def \bvolume#1{\textbf{#1}}\fi
\ifx \byear  \undefined \def \byear#1{#1}\fi
\ifx \bissue  \undefined \def \bissue#1{#1}\fi
\ifx \bfpage  \undefined \def \bfpage#1{#1}\fi
\ifx \blpage  \undefined \def \blpage #1{#1}\fi
\ifx \burl  \undefined \def \burl#1{\textsf{#1}}\fi
\ifx \doiurl  \undefined \def \doiurl#1{\url{https://doi.org/#1}}\fi
\ifx \betal  \undefined \def \betal{\textit{et al.}}\fi
\ifx \binstitute  \undefined \def \binstitute#1{#1}\fi
\ifx \binstitutionaled  \undefined \def \binstitutionaled#1{#1}\fi
\ifx \bctitle  \undefined \def \bctitle#1{#1}\fi
\ifx \beditor  \undefined \def \beditor#1{#1}\fi
\ifx \bpublisher  \undefined \def \bpublisher#1{#1}\fi
\ifx \bbtitle  \undefined \def \bbtitle#1{#1}\fi
\ifx \bedition  \undefined \def \bedition#1{#1}\fi
\ifx \bseriesno  \undefined \def \bseriesno#1{#1}\fi
\ifx \blocation  \undefined \def \blocation#1{#1}\fi
\ifx \bsertitle  \undefined \def \bsertitle#1{#1}\fi
\ifx \bsnm \undefined \def \bsnm#1{#1}\fi
\ifx \bsuffix \undefined \def \bsuffix#1{#1}\fi
\ifx \bparticle \undefined \def \bparticle#1{#1}\fi
\ifx \barticle \undefined \def \barticle#1{#1}\fi
\bibcommenthead
\ifx \bconfdate \undefined \def \bconfdate #1{#1}\fi
\ifx \botherref \undefined \def \botherref #1{#1}\fi
\ifx \url \undefined \def \url#1{\textsf{#1}}\fi
\ifx \bchapter \undefined \def \bchapter#1{#1}\fi
\ifx \bbook \undefined \def \bbook#1{#1}\fi
\ifx \bcomment \undefined \def \bcomment#1{#1}\fi
\ifx \oauthor \undefined \def \oauthor#1{#1}\fi
\ifx \citeauthoryear \undefined \def \citeauthoryear#1{#1}\fi
\ifx \endbibitem  \undefined \def \endbibitem {}\fi
\ifx \bconflocation  \undefined \def \bconflocation#1{#1}\fi
\ifx \arxivurl  \undefined \def \arxivurl#1{\textsf{#1}}\fi
\csname PreBibitemsHook\endcsname

\bibitem[\protect\citeauthoryear{González et~al.}{2017a}]{gonzalez2017quantification}
\begin{barticle}
\bauthor{\bsnm{González}, \binits{P.}},
\bauthor{\bsnm{Díez}, \binits{J.}},
\bauthor{\bsnm{Chawla}, \binits{N.}},
\bauthor{\bsnm{Coz}, \binits{J.J.}}:
\batitle{{Why is Quantification an Interesting Learning Problem?}}
\bjtitle{Progress in Artificial Intelligence}
\bvolume{6},
\bfpage{53}--\blpage{58}
(\byear{2017})
\doiurl{10.1007/s13748-016-0103-3}
\end{barticle}
\endbibitem

\bibitem[\protect\citeauthoryear{González et~al.}{2017b}]{gonzález2017}
\begin{barticle}
\bauthor{\bsnm{González}, \binits{P.}},
\bauthor{\bsnm{Castaño}, \binits{A.}},
\bauthor{\bsnm{Chawla}, \binits{N.V.}},
\bauthor{\bsnm{Coz}, \binits{J.J.}}:
\batitle{{A Review on Quantification Learning}}.
\bjtitle{ACM Computing Surveys}
\bvolume{50}(\bissue{5}),
\bfpage{74}--\blpage{17440}
(\byear{2017})
\doiurl{10.1145/3117807}
\end{barticle}
\endbibitem

\bibitem[\protect\citeauthoryear{Esuli et~al.}{2023}]{esuli2023learning}
\begin{bbook}
\bauthor{\bsnm{Esuli}, \binits{A.}},
\bauthor{\bsnm{Fabris}, \binits{A.}},
\bauthor{\bsnm{Moreo}, \binits{A.}},
\bauthor{\bsnm{Sebastiani}, \binits{F.}}:
\bbtitle{{Learning to Quantify}},
\bedition{1}st edn.
\bsertitle{The Information Retrieval Series}.
\bpublisher{Springer},
\blocation{Cham}
(\byear{2023}).
\doiurl{10.1007/978-3-031-20467-8}
\end{bbook}
\endbibitem

\bibitem[\protect\citeauthoryear{Curier et~al.}{2018}]{curier2018monitoring}
\begin{botherref}
\oauthor{\bsnm{Curier}, \binits{R.L.}},
\oauthor{\bsnm{De~Jong}, \binits{T.J.A.}},
\oauthor{\bsnm{Strauch}, \binits{K.}},
\oauthor{\bsnm{Cramer}, \binits{K.}},
\oauthor{\bsnm{Rosenski}, \binits{N.}},
\oauthor{\bsnm{Schartner}, \binits{C.}},
\oauthor{\bsnm{Debusschere}, \binits{M.}},
\oauthor{\bsnm{Ziemons}, \binits{H.}},
\oauthor{\bsnm{Iren}, \binits{D.}},
\oauthor{\bsnm{Bromuri}, \binits{S.}}:
{Monitoring Spatial Sustainable Development: Semi-Automated Analysis of Satellite and Aerial Images for Energy Transition and Sustainability Indicators}
(2018)
\doiurl{10.48550/arXiv.1810.04881}
\end{botherref}
\endbibitem

\bibitem[\protect\citeauthoryear{Rath et~al.}{2020}]{rath2020}
\begin{barticle}
\bauthor{\bsnm{Rath}, \binits{S.}},
\bauthor{\bsnm{Tripathy}, \binits{A.}},
\bauthor{\bsnm{Tripathy}, \binits{A.R.}}:
\batitle{{Prediction of New Active Cases of Coronavirus Disease (COVID-19) Pandemic Using Multiple Linear Regression Model}}.
\bjtitle{Diabetology \& Metabolic Syndrome}
\bvolume{14}(\bissue{5}),
\bfpage{1467}--\blpage{1474}
(\byear{2020})
\doiurl{10.1016/j.dsx.2020.07.045}
\end{barticle}
\endbibitem

\bibitem[\protect\citeauthoryear{Lakhan et~al.}{2020}]{lakhan2020}
\begin{barticle}
\bauthor{\bsnm{Lakhan}, \binits{R.}},
\bauthor{\bsnm{Agrawal}, \binits{A.}},
\bauthor{\bsnm{Sharma}, \binits{M.}}:
\batitle{{Prevalence of Depression, Anxiety, and Stress During COVID-19 Pandemic}}.
\bjtitle{Journal of Neurosciences in Rural Practice}
\bvolume{11}(\bissue{04}),
\bfpage{519}--\blpage{525}
(\byear{2020})
\doiurl{10.1055/s-0040-1716442}
\end{barticle}
\endbibitem

\bibitem[\protect\citeauthoryear{Chaudhry et~al.}{2021}]{chaudhry2021}
\begin{barticle}
\bauthor{\bsnm{Chaudhry}, \binits{H.N.}},
\bauthor{\bsnm{Javed}, \binits{Y.}},
\bauthor{\bsnm{Kulsoom}, \binits{F.}},
\bauthor{\bsnm{Mehmood}, \binits{Z.}},
\bauthor{\bsnm{Khan}, \binits{Z.I.}},
\bauthor{\bsnm{Shoaib}, \binits{U.}},
\bauthor{\bsnm{Janjua}, \binits{S.H.}}:
\batitle{{Sentiment Analysis of Before and After Elections: Twitter Data of US Election 2020}}.
\bjtitle{Electronics}
\bvolume{10}(\bissue{17}),
\bfpage{2082}
(\byear{2021})
\doiurl{10.3390/electronics10172082}
\end{barticle}
\endbibitem

\bibitem[\protect\citeauthoryear{Beijbom et~al.}{2015}]{beijbom2015}
\begin{barticle}
\bauthor{\bsnm{Beijbom}, \binits{O.}},
\bauthor{\bsnm{Edmunds}, \binits{P.J.}},
\bauthor{\bsnm{Roelfsema}, \binits{C.}},
\bauthor{\bsnm{Smith}, \binits{J.}},
\bauthor{\bsnm{Kline}, \binits{D.I.}},
\bauthor{\bsnm{Neal}, \binits{B.P.}},
\bauthor{\bsnm{Dunlap}, \binits{M.J.}},
\bauthor{\bsnm{Moriarty}, \binits{V.}},
\bauthor{\bsnm{Fan}, \binits{T.-Y.}},
\bauthor{\bsnm{Tan}, \binits{C.-J.}},
\bauthor{\bsnm{Chan}, \binits{S.}},
\bauthor{\bsnm{Treibitz}, \binits{T.}},
\bauthor{\bsnm{Gamst}, \binits{A.}},
\bauthor{\bsnm{Mitchell}, \binits{B.G.}},
\bauthor{\bsnm{Kriegman}, \binits{D.}}:
\batitle{{Towards Automated Annotation of Benthic Survey Images: Variability of Human Experts and Operational Modes of Automation}}.
\bjtitle{PloS One}
\bvolume{10}(\bissue{7}),
\bfpage{0130312}
(\byear{2015})
\doiurl{10.1371/journal.pone.0130312}
\end{barticle}
\endbibitem

\bibitem[\protect\citeauthoryear{Forman}{2005}]{forman2005}
\begin{bchapter}
\bauthor{\bsnm{Forman}, \binits{G.}}:
\bctitle{{Counting Positives Accurately Despite Inaccurate Classification}}.
In: \bbtitle{Lecture Notes in Computer Sciences},
pp. \bfpage{564}--\blpage{575}.
\bpublisher{Springer},
\blocation{Cham}
(\byear{2005}).
\doiurl{10.1007/11564096\_55}
\end{bchapter}
\endbibitem

\bibitem[\protect\citeauthoryear{Bross}{1954}]{bross1954}
\begin{barticle}
\bauthor{\bsnm{Bross}, \binits{I.D.J.}}:
\batitle{{Misclassification in 2X2 Tables}}.
\bjtitle{Biometrics}
\bvolume{10}(\bissue{4}),
\bfpage{478}--\blpage{486}
(\byear{1954})
\doiurl{10.2307/3001619}
\end{barticle}
\endbibitem

\bibitem[\protect\citeauthoryear{Dougherty}{2013}]{dougherty2013}
\begin{bbook}
\bauthor{\bsnm{Dougherty}, \binits{G.}}:
\bbtitle{{Pattern Recognition and Classification: An Introduction}},
\bedition{1}st edn.
\bpublisher{Springer},
\blocation{New York}
(\byear{2013}).
\doiurl{10.1007/978-1-4614-5323-9}
\end{bbook}
\endbibitem

\bibitem[\protect\citeauthoryear{Moreno-Torres et~al.}{2012}]{moreno2012}
\begin{barticle}
\bauthor{\bsnm{Moreno-Torres}, \binits{J.G.}},
\bauthor{\bsnm{Raeder}, \binits{T.}},
\bauthor{\bsnm{Alaíz-Rodríguez}, \binits{R.}},
\bauthor{\bsnm{Chawla}, \binits{N.V.}},
\bauthor{\bsnm{Herrera}, \binits{F.}}:
\batitle{{A Unifying View on Dataset Shift in Classification}}.
\bjtitle{Pattern Recognition}
\bvolume{45}(\bissue{1}),
\bfpage{521}--\blpage{530}
(\byear{2012})
\doiurl{10.1016/j.patcog.2011.06.019}
\end{barticle}
\endbibitem

\bibitem[\protect\citeauthoryear{Forman}{2006}]{forman2006}
\begin{bchapter}
\bauthor{\bsnm{Forman}, \binits{G.}}:
\bctitle{{Quantifying Trends Accurately despite Classifier Error and Class Imbalance}}.
\bpublisher{Association for Computing Machinery},
\blocation{New York}
(\byear{2006}).
\doiurl{10.1145/1150402.1150423}
\end{bchapter}
\endbibitem

\bibitem[\protect\citeauthoryear{Milli et~al.}{2013}]{milli2013}
\begin{bchapter}
\bauthor{\bsnm{Milli}, \binits{L.}},
\bauthor{\bsnm{Monreale}, \binits{A.}},
\bauthor{\bsnm{Rossetti}, \binits{G.}},
\bauthor{\bsnm{Giannotti}, \binits{F.}},
\bauthor{\bsnm{Pedreschi}, \binits{D.}},
\bauthor{\bsnm{Sebastiani}, \binits{F.}}:
\bctitle{{Quantification Trees}}.
In: \bbtitle{2013 IEEE 13th International Conference on Data Mining},
pp. \bfpage{528}--\blpage{536}.
\bpublisher{IEEE},
\blocation{Dallas}
(\byear{2013}).
\doiurl{10.1109/ICDM.2013.122}
\end{bchapter}
\endbibitem

\bibitem[\protect\citeauthoryear{Barranquero et~al.}{2013}]{barranquero2013}
\begin{barticle}
\bauthor{\bsnm{Barranquero}, \binits{J.}},
\bauthor{\bsnm{González}, \binits{P.}},
\bauthor{\bsnm{Díez}, \binits{J.}},
\bauthor{\bsnm{Del~Coz}, \binits{J.J.}}:
\batitle{{On the Study of Nearest Neighbor Algorithms for Prevalence Estimation in Binary Problems}}.
\bjtitle{Pattern Recognition}
\bvolume{46},
\bfpage{472}--\blpage{482}
(\byear{2013})
\doiurl{10.1016/j.patcog.2012.07.022}
\end{barticle}
\endbibitem

\bibitem[\protect\citeauthoryear{Esuli and Sebastiani}{2015}]{esuli2015}
\begin{barticle}
\bauthor{\bsnm{Esuli}, \binits{A.}},
\bauthor{\bsnm{Sebastiani}, \binits{F.}}:
\batitle{{Optimizing Text Quantifiers for Multivariate Loss Functions}}.
\bjtitle{Transactions on Knowledge Discovery from Data}
\bvolume{9}(\bissue{4}),
\bfpage{1}--\blpage{27}
(\byear{2015})
\doiurl{10.1145/2700406}
\end{barticle}
\endbibitem

\bibitem[\protect\citeauthoryear{Saerens et~al.}{2002}]{saerens2002}
\begin{barticle}
\bauthor{\bsnm{Saerens}, \binits{M.}},
\bauthor{\bsnm{Latinne}, \binits{P.}},
\bauthor{\bsnm{Decaestecker}, \binits{C.}}:
\batitle{{Adjusting the Outputs of a Classifier to New a Priori Probabilities: A Simple Procedure}}.
\bjtitle{Neural Computation}
\bvolume{14}(\bissue{1}),
\bfpage{21}--\blpage{41}
(\byear{2002})
\doiurl{10.1162/089976602753284446}
\end{barticle}
\endbibitem

\bibitem[\protect\citeauthoryear{Maletzke et~al.}{2019}]{maletzke2019}
\begin{bchapter}
\bauthor{\bsnm{Maletzke}, \binits{A.}},
\bauthor{\bsnm{Reis}, \binits{D.}},
\bauthor{\bsnm{Cherman}, \binits{E.}},
\bauthor{\bsnm{Batista}, \binits{G.}}:
\bctitle{{DyS: A Framework for Mixture Models in Quantification}}.
In: \bbtitle{Proceedings of the AAAI Conference on Artificial Intelligence},
vol. \bseriesno{33},
pp. \bfpage{4552}--\blpage{4560}
(\byear{2019}).
\doiurl{10.1609/aaai.v33i01.33014552}
\end{bchapter}
\endbibitem

\bibitem[\protect\citeauthoryear{King and Lu}{2008}]{king2008}
\begin{barticle}
\bauthor{\bsnm{King}, \binits{G.}},
\bauthor{\bsnm{Lu}, \binits{Y.}}:
\batitle{{Verbal Autopsy Methods with Multiple Causes of Death}}.
\bjtitle{Statistical Science}
\bvolume{23}(\bissue{1}),
\bfpage{78}--\blpage{91}
(\byear{2008})
\doiurl{10.1214/07-STS247}
\end{barticle}
\endbibitem

\bibitem[\protect\citeauthoryear{Forman}{2008}]{forman2008}
\begin{barticle}
\bauthor{\bsnm{Forman}, \binits{G.}}:
\batitle{{Quantifying Counts and Costs via Classification}}.
\bjtitle{Data Mining and Knowledge Discovery}
\bvolume{17}(\bissue{2}),
\bfpage{164}--\blpage{206}
(\byear{2008})
\doiurl{10.1007/s10618-008-0097-y}
\end{barticle}
\endbibitem

\bibitem[\protect\citeauthoryear{Kuha and Skinner}{1997}]{kuha1997}
\begin{bchapter}
\bauthor{\bsnm{Kuha}, \binits{J.}},
\bauthor{\bsnm{Skinner}, \binits{C.}}:
\bctitle{{Categorical Data Analysis and Misclassification}}.
In: \bbtitle{Surv. Meas. and Process Qual.},
pp. \bfpage{633}--\blpage{670}.
\bpublisher{John Wiley \& Sons, Ltd},
\blocation{Chichester, UK}
(\byear{1997}).
\bcomment{Chap. 28}.
\doiurl{10.1002/9781118490013.ch28}
\end{bchapter}
\endbibitem

\bibitem[\protect\citeauthoryear{Schumacher et~al.}{2023}]{schumacher2021}
\begin{botherref}
\oauthor{\bsnm{Schumacher}, \binits{T.}},
\oauthor{\bsnm{Strohmaier}, \binits{M.}},
\oauthor{\bsnm{Lemmerich}, \binits{F.}}:
A Comparative Evaluation of Quantification Methods
(2023).
\doiurl{10.48550/arXiv.2103.03223}
\end{botherref}
\endbibitem

\bibitem[\protect\citeauthoryear{Esuli et~al.}{2022}]{esuli2022detailed}
\begin{botherref}
\oauthor{\bsnm{Esuli}, \binits{A.}},
\oauthor{\bsnm{Moreo}, \binits{A.}},
\oauthor{\bsnm{Sebastiani}, \binits{F.}},
\oauthor{\bsnm{Sperduti}, \binits{G.}}:
{A Detailed Overview of LeQua@ CLEF 2022: Learning to Quantify}
(2022).
\url{https://api.semanticscholar.org/CorpusID:251471941}
\end{botherref}
\endbibitem

\bibitem[\protect\citeauthoryear{Kloos et~al.}{2021}]{kloos2021}
\begin{bchapter}
\bauthor{\bsnm{Kloos}, \binits{K.}},
\bauthor{\bsnm{Meertens}, \binits{Q.A.}},
\bauthor{\bsnm{Scholtus}, \binits{S.}},
\bauthor{\bsnm{Karch}, \binits{J.D.}}:
\bctitle{{Comparing Correction Methods to Reduce Misclassification Bias}}.
In: \bbtitle{Communications in Computer and Information Science},
vol. \bseriesno{1398},
pp. \bfpage{64}--\blpage{90}.
\bpublisher{Springer},
\blocation{Cham}
(\byear{2021}).
\doiurl{10.1007/978-3-030-76640-5\_5}
\end{bchapter}
\endbibitem

\bibitem[\protect\citeauthoryear{Kloos}{2021}]{kloos2021a}
\begin{barticle}
\bauthor{\bsnm{Kloos}, \binits{K.}}:
\batitle{{A New Generic Method to Improve Machine Learning Applications in Official Statistics}}.
\bjtitle{Statistical Journal of the IAOS}
\bvolume{37}(\bissue{4}),
\bfpage{1181}--\blpage{1196}
(\byear{2021})
\doiurl{10.3233/sji-210885}
\end{barticle}
\endbibitem

\bibitem[\protect\citeauthoryear{Tasche}{2017}]{tasche2017}
\begin{barticle}
\bauthor{\bsnm{Tasche}, \binits{D.}}:
\batitle{{Fisher Consistency for Prior Probability Shift}}.
\bjtitle{Journal of Machine Learning Research}
\bvolume{18}(\bissue{95}),
\bfpage{1}--\blpage{32}
(\byear{2017})
\doiurl{10.5555/3122009.3176839}
\end{barticle}
\endbibitem

\bibitem[\protect\citeauthoryear{Tasche}{2021}]{tasche2021}
\begin{botherref}
\oauthor{\bsnm{Tasche}, \binits{D.}}:
{Minimising Quantifier Variance Under Prior Probability Shift}
(2021).
\doiurl{10.48550/ARXIV.2107.08209}
\end{botherref}
\endbibitem

\bibitem[\protect\citeauthoryear{Meertens et~al.}{2022}]{meertens2021}
\begin{barticle}
\bauthor{\bsnm{Meertens}, \binits{Q.A.}},
\bauthor{\bsnm{Diks}, \binits{C.G.H.}},
\bauthor{\bsnm{Herik}, \binits{H.J.}},
\bauthor{\bsnm{Takes}, \binits{F.W.}}:
\batitle{{Improving the Output Quality of Official Statistics Based on Machine Learning Algorithms}}.
\bjtitle{Journal of Official Statistics}
\bvolume{38}(\bissue{2}),
\bfpage{485}--\blpage{508}
(\byear{2022})
\doiurl{10.2478/jos-2022-0023}
\end{barticle}
\endbibitem

\bibitem[\protect\citeauthoryear{Scholtus and van Delden}{2020}]{scholtus_accuracy_2020}
\begin{botherref}
\oauthor{\bsnm{Scholtus}, \binits{S.}},
\oauthor{\bsnm{Delden}, \binits{A.}}:
{On the accuracy of estimators based on a binary classifier}.
Statistics Netherlands
(2020).
\url{https://www.researchgate.net/profile/Arnout-Delden/publication/339129903_On_the_accuracy_of_estimators_based_on_a_binary_classifier/links/5e3f0159299bf1cdb918ef9d/On-the-accuracy-of-estimators-based-on-a-binary-classifier.pdf}
\end{botherref}
\endbibitem

\bibitem[\protect\citeauthoryear{Buonaccorsi}{2010}]{buonaccorsi_measurement_2010}
\begin{bbook}
\bauthor{\bsnm{Buonaccorsi}, \binits{J.P.}}:
\bbtitle{Measurement Error: Models, Methods, and Applications},
p. \bfpage{464}.
\bpublisher{{Chapman \& Hall/CRC}},
\blocation{Boca Raton, FL}
(\byear{2010}).
\doiurl{10.1201/9781420066586}
\end{bbook}
\endbibitem

\bibitem[\protect\citeauthoryear{Kloos et~al.}{2022}]{kloos2022unileiden}
\begin{bchapter}
\bauthor{\bsnm{Kloos}, \binits{K.}},
\bauthor{\bsnm{Meertens}, \binits{Q.A.}},
\bauthor{\bsnm{Karch}, \binits{J.D.}}:
\bctitle{{UniLeiden at LeQua 2022: The First Step in Understanding the Behaviour of the Median Sweep Quantifier using Continuous Sweep}}.
In: \bbtitle{Working Notes of the 2022 Conference and Labs of the Evaluation Forum (CLEF 2022), Bologna, Italy}
(\byear{2022}).
\burl{https://ceur-ws.org/Vol-3180/paper-150.pdf}
\end{bchapter}
\endbibitem

\bibitem[\protect\citeauthoryear{Kuhn and Wickham}{2020}]{tidymodels2020}
\begin{botherref}
\oauthor{\bsnm{Kuhn}, \binits{M.}},
\oauthor{\bsnm{Wickham}, \binits{H.}}:
{Tidymodels: a Collection of Packages for Modeling and Machine Learning Using Tidyverse Principles}.
(2020).
\url{https://www.tidymodels.org}
\end{botherref}
\endbibitem

\bibitem[\protect\citeauthoryear{Moreo et~al.}{2021}]{moreo2021quapy}
\begin{bchapter}
\bauthor{\bsnm{Moreo}, \binits{A.}},
\bauthor{\bsnm{Esuli}, \binits{A.}},
\bauthor{\bsnm{Sebastiani}, \binits{F.}}:
\bctitle{{QuaPy: A Python-Based Framework for Quantification}}.
In: \bbtitle{Proceedings of the 30th ACM International Conference on Information \& Knowledge Management}.
\bsertitle{CIKM '21},
pp. \bfpage{4534}--\blpage{4543}.
\bpublisher{Association for Computing Machinery},
\blocation{New York, NY, USA}
(\byear{2021}).
\doiurl{10.1145/3459637.3482015}
\end{bchapter}
\endbibitem

\end{thebibliography}
\end{document}